\documentclass{article}

\usepackage{times}
\usepackage{graphicx} 

\usepackage{amsmath,amssymb,stmaryrd}
\usepackage{amsthm}
\usepackage{enumitem}

\usepackage{caption}
\usepackage{subcaption}

\usepackage{natbib}

\usepackage{algorithm}
\usepackage{algorithmic}

\usepackage{hyperref}

\usepackage{fullpage}

\newtheorem{theorem}{Theorem}
\newtheorem{lemma}{Lemma}

\def\fro{{\scriptscriptstyle \mathrm{F}}}
\def\triple{{|\!|\!|}}
\def\ab{{\mathbf a}}

\def\xb{{\mathbf x}}

\def\zb{{\mathbf z}}
\def\bb{{\mathbf b}}

\def\wb{{\mathbf w}}

\def\eb{{\mathbf e}}
\def\ub{{\mathbf u}}
\def\vb{{\mathbf v}}

\def\alphab{{\boldsymbol\alpha}}
\def\betab{{\boldsymbol\beta}}

\def\Deltab{{\boldsymbol\Delta}}

\def\lambdab{{\boldsymbol\lambda}}
\def\mub{{\boldsymbol\mu}}

\def\oneb{{\mathbf 1}}

\def\Ib{{\mathbf I}}
\def\Ab{{\mathbf A}}
\def\Gb{{\mathbf G}}

\def\Real{{\mathbb{R}}}

\newcommand{\Int}[1]{\{1,\cdots,#1\}}

\def\Gcal{\mathcal{G}}
\def\Ocal{\mathcal{O}}
\def\Rcal{\mathcal{R}}
\def\Bcal{\mathcal{B}}
\def\Hcal{\mathcal{H}}

\def\Ical{\mathcal{I}}
\def\Acal{\mathcal{A}}
\def\Dcal{\mathcal{D}}
\def\Ecal{\mathcal{E}}

\def\Pcal{\mathcal{P}}

\def\Scal{\mathcal{S}}
\def\Xcal{\mathcal{X}}
\def\Lcal{\mathcal{L}}
\def\Qcal{\mathcal{Q}}
\def\dom{{\mathrm{dom}}}

\def\Exp{{\mathbb{E}}}

\def\defin{\triangleq}

\usepackage[colorinlistoftodos,bordercolor=orange,backgroundcolor=orange!20,linecolor=orange,textsize=scriptsize]{todonotes}
\newcommand{\be}{\mathbf{e}}

\newcommand{\bx}{\mathbf{x}}
\newcommand{\bal}{\boldsymbol{\alpha}}
\newcommand{\bI}{\mathbf{I}}
\newcommand{\bDel}{\boldsymbol{\Delta}}
\newcommand{\bmu}{\boldsymbol{\mu}}
\newcommand{\bW}{\mathbf{W}}

\newcommand{\cS}{\mathcal{S}}

\newcommand{\R}{\mathbb{R}}
\newcommand{\eqdef}{\triangleq}
\def\Tr{{\mathrm{Tr}}}

\def\hlb{{\hat{\lambdab}}}
\def\hxb{{\hat{\xb}}}
\def\hAb{{\hat{\Ab}}}

\title{Online optimization and regret guarantees \\for non-additive long-term constraints}

\author{Rodolphe Jenatton \\ 
Amazon.com \\ 
Berlin, Germany \\
\texttt{jenatton@amazon.com}
\and
Jim C. Huang \\ 
Amazon.com \\ 
Seattle, WA 98109 USA \\
\texttt{huangjim@amazon.com}
\and
Dominik Csiba \\
University of Edinburgh \\
Edinburgh, UK \\
\texttt{cdominik@gmail.com}
\and
Cedric Archambeau \\ 
Amazon.com \\ 
Berlin, Germany \\
\texttt{cedrica@amazon.com}
}

\begin{document} 

\maketitle

\begin{abstract} 
We consider online optimization in the 1-lookahead setting, where the objective does not decompose additively over the rounds of the online game. The resulting formulation enables us to deal with non-stationary and/or long-term constraints, which arise, for example, in online display advertising problems. We propose an online primal-dual algorithm for which we obtain dynamic cumulative regret guarantees. They depend on the convexity and the smoothness of the non-additive penalty, as well as terms capturing the smoothness with which the residuals of the non-stationary and long-term constraints vary over the rounds. We conduct experiments on synthetic data to illustrate the benefits of the non-additive penalty and show vanishing regret convergence on live traffic data collected by a display advertising platform in production.
\end{abstract} 
\section{Introduction}

Online optimization can be viewed as a sequential game where in each round $t \in\Int{T}$, we are required to play an action, represented by a vector $\mathbf{x}_t$, which takes values in a set of actions $\{\Xcal_t: \Xcal_t \subseteq \Real^d\}_{t=1}^T$. We then observe a reward $f_t(\mathbf{x}_t)$ as a function of the action we chose in round $t$. The goal is to generate a sequence of actions such that some measure of performance is for instance maximized over the course of $T$ rounds.

The performance metric typically adopted in the online learning framework is the \textit{cumulative regret}~\cite{Cesa-Bianchi2006}. Moreover, online learning traditionally assumes that the objective function has an \textit{additive} structure that nicely decomposes as a sum of regrets over the rounds of the optimization. This means that there is no coupling of $\xb_t$'s across successive rounds when the sets $\Xcal_t$ are decoupled across time. Hence, the challenge in online learning rather lies in the fact that $\xb_t$ must be estimated before having access to the reward function $f_t$, which corresponds to the so-called \textit{$0$-lookahead} setting~\cite{Buchbinder2012,Andrew2013}. In the particular case of dynamic regret with additive objective functions~\cite{Zinkevich2003,Cesa-Bianchi2012,Hall2013,Jadbabaie2015}, one seeks to analyze how, under various assumptions on $f_t$ and $\Xcal_t$, the sequence of rewards $f_t(\hxb_t)$ collected in an online fashion compares with the best sequence of rewards $f_t(\xb_t^\star)$ collected in hindsight. The dynamic regret over $T$ rounds is defined as follows:
\begin{equation}\label{eq:traditional_dynamic_regret}
\frac{1}{T} \sum_{t=1}^T f_t(\xb_t^\star)  - \frac{1}{T} \sum_{t=1}^T f_t(\hxb_t) ,
\end{equation}
with
$\frac{1}{T} \!\sum_{t=1}^T f_t(\xb_t^\star) \!\!=\!\! {\max}_{\xb_1 \in \Xcal_1,\dots,\xb_T \in \Xcal_T}\!\frac{1}{T}\! \sum_{t=1}^T f_t(\xb_t) .$
Unlike static regret analysis~\cite{Zinkevich2003}, where we compare against a \textit{single} best action $\xb^\star$ in retrospect over all rounds played, the best dynamic comparator, as its name indicates, need not be identical for all $t \in \Int{T}$.

In this work, we consider a setting similar to online learning with dynamic regret, 
but we do not assume that successive actions $\xb_t$ are decoupled.
We provide novel regret bounds for the case where rewards are additive and the regrets are non-additive as the total cost over $T$ rounds is non-decomposable. There is a small body of recent work that study non-additive regrets. For example, \citet{Rakhlin2010} consider a wide class of non-additively decomposable objective functions in a 0-lookahead setting, covering for instance the problems of Blackwell's approachability, the calibration of forecasters and the global cost online learning game from~\citet{Even-Dar2009}. More recently, \citet{Kar2014} handle some specific form of non-decomposability in relation with the online optimization of metrics such as the precision at $k$.  Closer in spirit to our work is the one by \citet{Agrawal2015}, who provide expected regret bounds in the case where non-additive costs depend only on the empirical mean of the actions. %As it will be made more formal in the sequel, our 

Metrical task systems (MTS)~\cite{Borodin1992} provide an alternative analysis framework for online optimization. They consider \textit{movement costs} that penalize variations of $\xb_t$'s \textit{across time} and rely on competitive analysis~\cite{Borodin2005}. Similar to online learning with dynamic regret, the performance of the online optimization algorithm is compared to the best sequence of actions. However, instead of measuring cumulative regret, competitive analysis adopts a \emph{multiplicative} metric, known as the competitive ratio. We refer the interested reader to \cite{Buchbinder2012,Andrew2013} for a detailed discussion. Moreover, the online game in MTS differs from the online learning setup in that it follows the \textit{$1$-lookahead} setting where the player has access to the reward function $f_t$ \textit{before} estimating $\xb_t$. The typical instantiation of the movement cost in MTS is a total variation penalty $\sum_{t=1}^{T-1} \| \xb_t - \xb_{t+1} \|_2$, that can also be defined with non-Euclidean norms~\cite{Bera2013}. 
While the movement cost introduces dependencies across actions $\xb_t$ in successive rounds, we note that, to the best of our knowledge, previous work only considered movement costs with an additive structure. We depart from this approach by considering non-additive penalties.  

The aforementioned frameworks have been successfully applied to derive and analyze online optimization algorithms in  for example online routing \cite{Awerbuch2008}, process migration of servers \cite{Borodin2005} and portfolio allocation \cite{portfolio}. However, the assumption that regrets or movement costs have an additive structure is restrictive in practice. In particular, we consider the problem of \textit{online ad allocation}, which is at the core of modern display advertising systems. The online ad allocation \cite{rtb} problem consists of sequentially allocating ad impressions (encoded by $\mathbf{x}_t$) to a large number of competing ad slots across a large number of websites and mobile apps, subject to a variety of advertiser objectives and constraints. Advertisers will typically expect that in solving the ad allocation problem, we maximize a measure of ad performance (or advertiser welfare), subject to constraints on user targeting and constraints on ad delivery (e.g., spend as close to 100\% of an ad's budget as possible over $T$ rounds). The online ad allocation problem is characterized by non-stationarity in constraints (e.g., for a budget-constrained allocation, the amount of budget consumed per round varies dynamically). Moreover, in such an application, constraints are measured by the advertiser only in a long-term sense (e.g., budget consumed by the end of an advertising campaign), which requires the use of non-additive constraints.

Hence, the online ad allocation problem can be viewed as a hybrid between online learning with dynamic regret and MTS. It falls into the 1-lookahead setting and requires non-additive constraints, but there is no practically-justifiable concept of movement cost on $\hat{\mathbf{x}}_t$ that can be applied to account for non-additive constraints. Moreover, the movement cost is not a sensible penalty to capture the fact that one would like to show different ads in successive rounds. Finally, online ad allocation is concerned with satisfying \emph{long-term} constraints, which means that the cumulative constraint violations resulting from the sequence of vectors $\{\mathbf{x}_t\}_{t=1}^{T}$ should not exceed a certain amount by the final round $T$. The previous approaches proposed in the online learning literature~\cite{Mahdavi2012,Agrawal2015} are not suitable for our use-case since they handle stationary long-term constraints with static regret guarantees in the 0-lookahead setting. And again, this type of constraints is  difficult to encode as movement costs.

\paragraph{Contributions.} Within the 1-lookahead setting, we study an online optimization problem with dynamic regret defined in terms of a sum of concave rewards and a penalty that does not decompose additively over the rounds of the online game. Non-additive penalties are suitable for modeling \textit{non-stationary} and/or  \textit{long-term} constraints, which are of practical importance in display advertising. The resulting formulation is, to the best of our knowledge, novel and extends the work of~\citet{Mahdavi2012} and~\citet{Agrawal2015}. We further provide dynamic cumulative regret guarantees, showing that vanishing regret is driven by the convexity and the smoothness of the non-additive penalty, along with the smoothness with which the residuals vary over the rounds. Hence, the structure of our bound echoes results derived in previous work within the context of 0-lookahead dynamic regret analysis~\cite{Zinkevich2003}. Finally, we illustrate our methodological contribution by conducting experiments on synthetic data to validate the benefits of non-additive penalties and study the convergence of the cumulative regret on live traffic data collected by a display advertising platform in production.

\section{Problem statement}

Motivated by the practical realities of the online ad allocation problem, we focus on a class of online optimization problems in which the actions played over time must not only maximize some rewards but also minimize \emph{long-term} constraints that are penalized by some \textit{non-additive} error term through the function $\Ecal$. Formally, for any $(\xb_1,\dots,\xb_T) \in \Xcal_1 \times \dots \times \Xcal_T$, we study the online maximization of objective functions of the form
\begin{equation}\label{eq:primal_cost}
\!\!\!\Pcal(\xb_1,\dots,\xb_T)\! \defin\! \frac{1}{T}
\!\sum_{t=1}^T f_t(\xb_t)
-
\Ecal
\bigg(
\!\frac{1}{T}\sum_{t=1}^T \Ab_t \xb_t - \bb_t\!\!
\bigg) ,
\end{equation}
where the matrices $\Ab_t \in \Acal \subseteq \Real^{m \times d}$ and vectors $\bb_t \in \Real^m$ encode the non-stationary, long-term constraints. %The penalty function $\cal{E}$ is assumed to be convex or strongly convex.
Although the sets $\Xcal_t$'s depend on $t$, they are assumed to be fully available at each round $t$ and do not take part in the online game.

Our goal is to derive an online algorithm and prove that the sequence $(\hxb_1,\dots,\hxb_T) \in  \Xcal_1 \times \dots \times \Xcal_T$ it generates is guaranteed to satisfy a regret bound of the form
\begin{equation}\label{eq:informal_regret_bound}
\Pcal^\star  -  \Pcal(\hxb_1,\dots,\hxb_T)
\leq \Rcal_T + \Scal_{\eb} + \Scal_{\Ab} , 
\end{equation}
where we have defined the optimal (primal) objective
\begin{equation}\label{eq:argmax_offline}
\Pcal^\star \defin\!\!\!\! \!\!\max_{\xb_1 \in \Xcal_1,\dots,\xb_T \in \Xcal_T} \!\!\! \Pcal(\xb_1,\dots,\xb_T) \! = \! \Pcal(\xb_1^\star,\dots,\xb_T^\star).
\end{equation}
The regret bound in~(\ref{eq:informal_regret_bound}) is comprised of three terms, each capturing different aspects of the problem.

The first term $\Rcal_T$ quantifies the contribution due to the application of an online projected subgradient method in the dual problem for~(\ref{eq:primal_cost}), as described in Section~\ref{sec:lagrangian}.

The second term $\Scal_{\eb}$ quantifies the impact of the smoothness with which the sequences of $m$-dimensional \textit{error vectors} $\{\eb_t^\star\}_{t=1}^T$ evolve over time, where $
\eb_t^\star \defin \Ab_t \xb_t^\star - \bb_t.$ In this respect, the term $\Scal_\eb$ will be reminiscent of the guarantees traditionally obtained with dynamic regret analysis ~\cite{Cesa-Bianchi2012}, 
with the key exception that the smoothness is not based on the variables $\xb_t^\star$ themselves, but rather on the error vectors $\eb_t^\star$. This term also has a natural interpretation and relevance in the online advertising setting, as discussed in Section~\ref{sec:online_ad_allocation_exp}. The derivation for $\Scal_\eb$ will be the topic of Section~\ref{sec:error_smoothness}.\par
The third and last term $\Scal_{\Ab}$ models the impact of having to estimate constraint matrices $\Ab_t$ via $\hAb_t$, as in practice $\Ab_t$ is not known at the time when the action $\xb_t$ is played. This aspect is also motivated by practical aspects of the online ad allocation problem. As it will be made clear in Section~\ref{sec:estimate_A}, $\Scal_\Ab$ will depend on how smoothly the constraint matrices $\Ab_t$'s evolve over time.
Given the above, our methodological contributions can thus be viewed as both a regret analysis with dynamic comparators for non-additive objective functions, and a competitive analysis evaluated through a regret criterion -- or more precisely, a \textit{competitive difference} criterion, as defined in \citet{Andrew2013} -- where the service cost is $f_t$ and our movement cost $\Ecal$ is convex and non-additive. We will interchangeably refer to both competitive difference and dynamic regret in the sequel.

\section{Analysis}\label{sec:analysis}
We describe in this section the main components of our analysis, starting with the saddle point formulation.
The proofs of the results are relegated to the Appendix.
\subsection{Linearization and Lagrangian formulation}\label{sec:lagrangian}
Since the main challenge of our analysis lies in the fact that $\Ecal$ breaks the additive structure of~(\ref{eq:traditional_dynamic_regret}), a natural strategy is therefore to linearize $\Ecal$ via Fenchel conjugacy, similar to~\cite{Agrawal2015}. We define the Fenchel conjugate of $\Ecal$ as $\Ecal^\star(\lambdab) \defin \sup_{\mathbf{z} \in \dom(\Ecal)} \{ \lambdab^\top\mathbf{z} - \Ecal(\mathbf{z})\}$~\cite{Boyd2004}, and from now on, we assume that\par
\textbf{(A1)} The function $\Ecal$ is convex.\footnote{When we refer to convex/concave functions, we implicitly assume closed proper convex/concave functions~\cite{Boyd2004}.}
\par
\textbf{(A2)} The Fenchel conjugate $\Ecal^\star$ of $\Ecal$ has its domain 
$$
\Lambda \defin \dom(\Ecal^\star)  = \{  \lambdab \in \Real^m : \Ecal^\star(\lambdab) < +\infty  \}
$$
that is compact so that there exists $R_\lambda\! \defin\displaystyle\!  \max_{\lambdab \in \Lambda} \| \lambdab \|_2 < +\infty$.\par
\textbf{(A3)} For any $t\in\Int{T}$, the function $f_t$ is concave, and the set $\Xcal_t$ is compact and convex with $\Xcal_t \subseteq \dom(f_t)$,  
so there exists $R_x\!\defin\! \max_{\substack{t\in  \Int{T}, \xb_t \in \Xcal_t}}\! \|\xb_t\|_2 \!< \!+\infty$.
\par
Equipped with those assumptions, we introduce for any $\xb_t \in \Xcal_t$, $\lambdab \in \Lambda$ the Lagrangian function for round $t$:
\begin{equation}\label{eq:lag}
\Lcal_t(\xb_t,\lambdab) \defin f_t(\xb_t) - \lambdab^\top (\Ab_t \xb_t - \bb_t) + \Ecal^\star(\lambdab) .
\end{equation}
It can be observed that $\min_{\lambdab \in \Lambda} \frac{1}{T} \sum_{t=1}^T \Lcal_t(\xb_t,\lambdab) = \Pcal(\xb_1,\dots,\xb_T)$.

In the sequel, we shall refer to $\lambdab \in \Real^m$ as dual variables. Following previous work on online saddle point optimization~\cite{Mahdavi2012,Koppel2014}, our procedure will alternate between maximizing the Lagrangian with respect to the primal variable $\xb_t$ and minimizing with respect to the dual variable $\lambdab$. On the primal side, we note that we are interested in the \textit{1-lookahead} setting, where we have access to $\Lcal_t$ before computing our estimate $\hxb_t$, whereas on the dual side, our sequence of $\hlb_t$'s will be generated according to the \textit{0-lookahead} setting. This follows the sequencing of events in practice for display advertising, where the dual variables can only be updated after we observe $\mathbf{A}_{t}$, which is only observed after we play $\hat{\mathbf{x}}_t$.
Before summarizing the procedure in Algorithm~\ref{alg:with_known_At}, we specify a key computational assumption with respect to the primal variables:

~

\textbf{(A4)} For any $\lambdab \in \Lambda$, we can efficiently compute $\arg\!\max_{\xb_t \in \Xcal_t}  \Lcal_t(\xb_t, \lambdab)$.

~

In this paper, we will be primarily interested in problems where the above assumption \textbf{(A4)} holds (e.g., online ad allocation problems consisting of maximizing linear functions over simplices), so that the complexity and tractability of our proposed algorithm will mostly depend on structural properties of $f_t$ and $\Xcal_t$. Having laid out assumptions, we present our online algorithm in Algorithm~\ref{alg:with_known_At}, which makes use of an online projected subgradient method (OPSM)~\cite{Zinkevich2003,Hazan2007} with respect to the dual variables. The notation $\Pi_\Qcal$ refers to the Euclidean projection onto the set $\Qcal$. Interestingly, while the analysis of the primal objective~(\ref{eq:primal_cost}) requires a comparison to $T$ different optimal primal variables (i.e., measuring performance via dynamic regret), the online optimization with respect to the dual variables reduces to a static regret analysis that compares with an unique optimal dual variable $\lambdab^\star$. A related observation was exploited by~\cite{Shalev-Shwartz2006,Shalev-Shwartz2009}, but following an opposite route, i.e., analyzing the primal objective using static regret, while leveraging the dynamic regret structure of the dual optimization problem. Given Algorithm \ref{alg:with_known_At}, we now
study the impact of applying OPSM for our problem in Lemma \ref{lem:standard_static_regret}, subject to an additional assumption about the boundedness of the subgradients of $\Lcal_t(\xb_t,\cdot)$:

~

\textbf{(A5)} {\small There exists} $G>0, \displaystyle \max_{\substack{t\in\Int{T},\\\xb_t \in \Xcal_t,\lambdab \in \Lambda}} \| \nabla_\lambdab  \Lcal_t(\xb_t,\lambdab) \|_2 \leq G$.

\begin{algorithm}[t]
   \caption{Online saddle point optimization}
   \label{alg:with_known_At}
\begin{algorithmic}
   \STATE {\bfseries Input:} Initial dual variable $\hlb_1$, step sizes $\eta_t$
   \FOR{$t=1$ {\bfseries to} $T$}
   \STATE Receive $\Lcal_t$   \texttt{  //1-lookahead setting}
   \STATE Compute $\hxb_t \in \arg\!\max_{\xb_t \in \Xcal_t}  \Lcal_t(\xb_t, \hlb_t)$
   \STATE Update $\hlb_{t+1} = \Pi_{\Lambda}[ \hlb_{t} - \eta_t \nabla_\lambdab \Lcal_t(\hxb_t, \hlb_t) ]$
   \ENDFOR
\end{algorithmic}
\end{algorithm}
\begin{lemma}\label{lem:standard_static_regret}
Let assumptions \textbf{(A1)-(A3)} and \textbf{(A5)} hold. Let $\kappa \geq 0$ be the strong convexity parameter of $\Lcal_t(\hxb_t,\cdot)$ with respect to the $\ell_2$ norm. It holds for any $\lambdab \in \Lambda$, and sequences $\{\hxb_t,\hlb_t\}_{t=1}^T$ generated by~Algorithm~\ref{alg:with_known_At} that
$$
\frac{1}{T}\sum_{t=1}^T {\cal L}_t(\hxb_t,\hat{\boldsymbol{\lambda}}_t) - {\cal L}_t(\hat{\mathbf{x}}_t,\boldsymbol{\lambda}) \leq {\Rcal}_T,
$$
where if $\Lcal_t(\hxb_t,\cdot)$ is strongly convex ($\kappa \!> \!0$), we take the step size $\eta_t \!=\! \frac{1}{\kappa t}$ and we have $\Rcal_T\! =\! \frac{G^2}{2\kappa T} \log(e T)$, while if $\Lcal_t(\hxb_t,\cdot)$ is convex ($\kappa = 0$), we can choose the step size $\eta_t = \frac{2R_\lambda}{G\sqrt{T}}$ leading to $\Rcal_T = \frac{2R_\lambda G}{\sqrt{T}}$.
\end{lemma}
The two claims can be found respectively in \citep{Zinkevich2003,Hazan2007}, a short proof is given in the Appendix for self-containedness. 
For the convex case, the step size $\eta_t$ depends on the time horizon $T$, and we can use the doubling trick (e.g., see Section 2.3.1 in~\cite{Shalev-Shwartz2011} to remove this dependency.

Having presented the component of the regret bound that comes from the use of OPSM, we next examine why the contribution of $\Rcal_T$ is not sufficient by itself to control the regret bound~(\ref{eq:informal_regret_bound}).
\subsection{Controlling the worst sequence of dual variables}\label{sec:error_smoothness}
Given Assumptions~\textbf{(A1)-(A3)}, there do exist some optimal primal variables $\{\xb_t^\star\}_{t=1}^T$ as defined in~(\ref{eq:argmax_offline}).
Let us denote an optimal dual variable given the primal variables $\{\xb_t^\star\}_{t=1}^T$ as
$$
\lambdab^\star \in \arg\! \min_{\lambdab \in \Lambda} \frac{1}{T}\sum_{t=1}^T \Lcal_t(\xb_t^\star,\lambdab),
$$
and define $\hlb$ equivalently for the sequence $\{\hxb_t\}_{t=1}^T$ where the existence of both $\lambdab^\star$ and $\hlb$ are guaranteed by Assumptions~\textbf{(A1)-(A2)}.

From the perspective of deriving an upper bound on $\Pcal^\star - \Pcal(\hxb_1,\dots,\hxb_T)$, we observe that $\Lcal_t(\xb_t^\star,\lambdab^\star) - \Lcal_t(\hxb_t,\hlb)$ can be rewritten as 
$$
(\Lcal_t(\xb_t^\star,\lambdab^\star) - \Lcal_t(\xb_t^\star,\hlb_t)) + (\Lcal_t(\xb_t^\star,\hlb_t) -  \Lcal_t(\hxb_t,\hlb)),
$$
which in turn, using \textbf{(A4)}, is upper-bounded by
\begin{equation}\label{eq:definition_gcal}
\hspace*{-0.2cm}\underbrace{(\Lcal_t(\xb_t^\star,\lambdab^\star) \!-\! \Lcal_t(\xb_t^\star,\hlb_t))}_{\defin \Gcal_t(\hlb_t)} \!+\! \underbrace{(\Lcal_t(\hxb_t,\hlb_t) \!-\!  \Lcal_t(\hxb_t,\hlb))}_{C},
\end{equation} where the second term above is readily obtained by the definition of $\hat{\mathbf{x}}_t \in \arg\!\max_{\xb_t\in\Xcal_t} {\cal L}_t(\mathbf{x},\hlb_t)$. This observation is useful in several respects: while the term $C$ can be directly bounded as a result of Lemma~\ref{lem:standard_static_regret}, the term $\Gcal_t(\hlb_t)$, which refers to the gap we incur as a result of the sequentially-generated $(\hlb_1,\dots,\hlb_T)$, is not accounted for by Lemma~\ref{lem:standard_static_regret}, and so it needs to be controlled differently. Notably, $\hlb_t~\mapsto~\Gcal_t(\hlb_t)$ is \textit{concave} over $\Lambda$, and so is $(\hlb_1,\dots,\hlb_T)~\mapsto~\frac{1}{T} \sum_{t=1}^T\Gcal_t(\hlb_t)$ over $\Lambda^T$: this implies that we can cast the problem of controlling the worst sequence of dual variables $(\hlb_1,\dots,\hlb_T)$ as a concave maximization problem.

In order to obtain a meaningful upper bound on $\Gcal_t(\hlb_t)$ and $\frac{1}{T} \sum_{t=1}^T\Gcal_t(\hlb_t)$, we leverage the fact that we are not dealing with \textit{any} general sequence $(\hlb_1,\dots,\hlb_T)\! \in \! \Lambda^T$, but sequences possibly output by~Algorithm~\ref{alg:with_known_At}. More precisely, we make this characterization by noting that OPSM generates successive $\hlb_t$ and $\hlb_{t+1}$ estimates
whose differences $\| \hlb_t -  \hlb_{t+1} \|_2$ are controlled as a function of the step size $\eta_t$~\cite{Andrew2013}. Following~\cite{Zinkevich2003}, we therefore introduce the convex set 
\begin{equation}\label{eq:set_lambda}
\hspace*{-0.25cm}\Lambda_{T, \varepsilon}\! \defin\! \bigg\{  \! (\lambdab_1,\dots,\lambdab_T) \in \Lambda^T \! : \! \sum_{t=1}^{T-1} \|  \lambdab_t - \lambdab_{t+1}  \|_2\!  \leq \! \varepsilon \! \bigg\},
\end{equation} which will be useful shortly in deriving an upper bound for $\Scal_e$. In particular, assuming \textbf{(A5)} holds and setting $\varepsilon = G \sum_{t=1}^T \eta_t$, for any sequence $(\hlb_1,\dots,\hlb_T)\! \in \! \Lambda^T$ generated by~Algorithm~\ref{alg:with_known_At}, we have $(\hlb_1,\dots,\hlb_T)\! \in \! \Lambda_{T, \varepsilon}$ and
\begin{equation}\label{eq:definition_scal_e}
\frac{1}{T} \sum_{t=1}^T\Gcal_t(\hlb_t) \leq 
\Scal_{\eb} \defin \!\!\!\!
\max_{(\lambdab_1,\dots,\lambdab_T) \in \Lambda_{T, \varepsilon}} \frac{1}{T} \sum_{t=1}^T \Gcal_t(\lambdab_t).
\end{equation}
We next turn to some lemmas that make the expression of $\Scal_\eb$ more explicit by leveraging duality arguments. To this end, we introduce some additional notation: first, we rewrite the total-variation constraint from~(\ref{eq:set_lambda}) as
$$
\sum_{t=1}^{T-1} \|  \lambdab_t - \lambdab_{t+1}  \|_2 = \Omega_{1/2}(  (  \Deltab \otimes \Ib ) \underline{\lambdab}  )
$$
where $\underline{\lambdab} \in \Real^{T\cdot m}$ stands for the vector formed by stacking the $m$-dimensional $\lambdab_t$'s, $\otimes$ is the Kronecker product, $\Deltab \in \Real^{(T-1)\times T}$ is the discrete 1-dimensional gradient matrix,\footnote{The matrix $\Deltab \in \Real^{(T-1)\times T}$ contains two non-zero entries per row, with $\Delta_{t,t} =1$ and $\Delta_{t,t+1} =-1$ for $t\in\Int{T-1}$.} while $\Omega_{1/2}$ is the $\ell_1/\ell_2$ norm with dual norm $\Omega_{\infty/2}$, such that
$$
\Omega_{1/2}(\ub) \defin\!\ \! \sum_{t=1}^{T-1} \|\ub_t\|_2\ \ \text{ and }\ \ \Omega_{\infty/2}(\vb) \defin\!\!\! \!\!\max_{t\in\Int{T-1}}\!\! \|\vb_t\|_2,
$$
for any vectors $\ub\! =\! [\ub_1,\dots,\ub_{T-1}]$ and $\vb \!=\! [\vb_1,\dots,\vb_{T-1}]$ in $\Real^{(T-1)\cdot m}$. We now show that $\Scal_\eb$ can conveniently be expressed as a minimization problem.
\begin{lemma}\label{lem:scal_e_as_min} Let $\eb^\star_t = \Ab_t \xb_t^\star - \bb_t$. It holds that
$$
\Scal_\eb
=\!\!\! 
\min_{\alphab \in \Real^{(T-1)\cdot m}}\!\!\!\!\! \Scal_\eb(\alphab)
$$
where we have introduced the convex function $\alphab \mapsto \Scal_\eb(\alphab)$
$$ 
\frac{1}{T}\!\! \sum_{t=1}^T
\Ecal \Big(\! 
\eb^\star_t - \big[  (  \Deltab \otimes \Ib )^\top \alphab  \big]_t
\!\Big)
- 
\Ecal \Big( 
\frac{1}{T} \sum_{t=1}^T \eb^\star_t 
\Big)
+
\frac{\varepsilon}{T} \Omega_{\infty/2}(\alphab).
$$
\end{lemma}

The proof is given in the Appendix and relies on duality arguments. We note at this juncture that our analysis can similarly handle other penalties related to the total-variation chosen in~(\ref{eq:set_lambda}), for instance $\|  (  \Deltab \otimes \Ib ) \underline{\lambdab}  \|_\fro^2$ (we omit the details owing to space limitations).
\par
The expression for $\Scal_\eb$ provided by Lemma~\ref{lem:scal_e_as_min} (the result of a \textit{minimization} problem), makes it possible to obtain an upper bound for any candidate vector $\alphab \in \Real^{(T-1)\cdot m}$. We propose below one such instantiation that highlights how $\Scal_\eb$ depends on the smoothness with which the sequence $\{\eb_{t}^\star\}_{t=1}^T$ varies over time:

\begin{lemma}\label{lem:upperbound_scal_e}
Let $\eb^\star_t = \Ab_t \xb_t^\star - \bb_t$ be stacked in the vector $\underline{\be}^\star \in \Real^{Tm}$.
The term $\Scal_\eb$ is upper-bounded by
\[\cS_\be \leq \frac{\epsilon}{T} \max_{t \in \{1, \dots, T-1\}} \Psi_t(\underline{\be}^\star)
\quad
\text{where}
\quad
\Psi_t(\underline{\be}^\star) \defin \bigg\| \sum_{j = 1}^t \frac{T-t}{T} \be_j^\star - \sum_{j = t+1}^T \frac{t}{T} \be_j^\star \bigg\|_2.
 \]
\end{lemma}
The proof of this lemma can be found in the Appendix. 
We can see from Lemma~\ref{lem:upperbound_scal_e} that $\Scal_\eb$ captures both (a) the cumulative variations of $\{\eb_t^\star\}_{t=1}^T$ through $\Psi_t(\underline{\be}^\star)$, modulated by the average of the step sizes $\varepsilon/T = G\sum_{t=1}^T \eta_t/T$, and (b) the worst of the constraint violations that surfaces via $G$ (see the definition of Assumption \textbf{(A5)}).

To intuitively understand the effect of $\Psi_t$, we can first observe that if the residual vectors $\{\eb_t^\star\}_{t=1}^T$ are constant over time, that is $\underline{\be}^\star = c\cdot \oneb$ for some scalar $c$, then the terms $\Psi_t(\underline{\be}^\star) $ vanish for all $t \in \{1,\dots,T\}$, so that $\Scal_\eb$ does not contribute to the regret guarantees in this case.
To get a better sense of the impact of $\Psi_t$ beyond the case where the residual vectors are perfectly constant, we now assume that  $\{\eb_t^\star\}_{t=1}^T$ are independent random (sub-Gaussian) vectors. It is important to stress the fact that our analysis and our main theorem (see Theorem~\ref{thm:main}) hold in absence of any stochastic assumptions, but we only momentarily consider random residuals in order to gain insight into how the term $\Scal_\eb$ can scale in more realistic scenarios beyond the case of constant residual vectors: 
\begin{lemma} \label{lem:concentration_bound}
Let $\underline{\be}^\star \in \Real^{Tm}$ be a random vector, such that for some $\bmu \in \R^{mT}$ and $\sigma \geq 0$ \begin{equation} \label{eq:assumption_randome}
\Exp [\exp (\bal^\top (\bx - \bmu))] \leq \exp (\|\bal\|^2 \sigma^2/2)
\end{equation} holds for every $\bal \in \R^{m T}$ . Let $Z_t \in \R_+$ and $\omega_{\bmu} \in \R_+$ be defined by 
\begin{equation} \label{def:Zt}
Z_t \eqdef \Psi_t(\underline{\be}^\star) 
\qquad
\mbox{and} 
\qquad 
\omega_{\bmu} \eqdef \max_{t \in \{1, \dots, T-1\}} [\Psi_t(\mub)]^2.
\end{equation} Then we have \[ \Exp \left[ \max_{t \in \{1, \dots, T-1\}} Z_t  \right] \leq \sqrt{\sigma^2 m T + \omega_{\bmu}} + \sqrt{\sigma^2 m T \log(T)}. \]
\end{lemma}

The proof of this lemma is given in the Appendix. Some comments are in order. Lemma~\ref{lem:concentration_bound} shows that when $\underline{\be}^\star \in \Real^{Tm}$ is a random (sub-Gaussian) vector---which notably covers the cases where the entries of $\underline{\be}^\star$ are independent bounded, or Gaussian, random variables---then, the term $\Scal_\eb$ scales (in expectation) as
$$
\Ocal
\Big(
\frac{\varepsilon}{T} 
\Big[
\sqrt{\sigma^2 m T + \omega_{\bmu}} + \sqrt{\sigma^2 m T \log(T)}
\Big]
\Big).
$$
Assuming that $\omega_{\bmu} = \Ocal(\sigma^2 m T \log(T))$, we can see that it is sufficient to have $\varepsilon =o(\sqrt{T/(\sigma^2 m \log(T))})$ in order to guarantee that $\Scal_\eb$ vanishes. This condition is for instance satisfied in the setting where $\Ecal$ has Lipschitz continuous gradients with parameter $L > 0$, for which we can take $\eta_t = L/t$ and $\varepsilon = \Ocal(G\log(T))$.
Interestingly, the constant term $\mub$ appears only through $\omega_{\bmu}$, and hence $\Psi_t(\mub)$, so that, as discussed previously, a constant (even non-zero) vector $\mub$ leads to $\omega_{\bmu}=0$. We therefore see that $\omega_\mub$ penalizes by how much we deviate from a constant mean vector.

Having derived the terms ${\cal R}_T$ and $\Scal_e$ in the cumulative regret bound, we now turn to the description of the last term $\Scal_\Ab$ of our regret guarantee~(\ref{eq:informal_regret_bound}).

\subsection{Estimating the matrices $\{\Ab_t\}_{t=1}^T$}\label{sec:estimate_A}
So far we have assumed that we have access to $\mathbf{A}_t$ at each round, when in practice we only get to observe $\mathbf{A}_t$ \textit{after} we take an action $\hxb_t$. For example, in the online ad allocation problem, the amount of money to be charged to the advertiser for a single ad impression is revealed only after we have made a decision on which ad to allocate for the impression, i.e., once we have computed $\hat{\mathbf{x}}_t$. We now address the cost incurred in having to estimate the constraint matrices $\Ab_t$ before playing $\hat{\mathbf{x}}_t$. To this end, we assume that in addition to the dual variables being bounded with radius $R_{\lambda}$, we have 

~

\textbf{(A6)} The set $\Acal$ is convex and bounded; in particular, there exists $R_A \defin \max_{\Ab \in \Acal} \|\Ab\|_\fro < +\infty$.

~

\noindent Since the choice of $\hxb_t$ now depends on the estimate $\hAb_t$, we introduce the estimated Lagrangian
\begin{equation}\label{eq:hatlcal}
\hat{\Lcal}_t(\xb_t,\lambdab) \defin
f_t(\xb_t) - \lambdab^\top(\hAb_t \xb_t - \bb_t) + \Ecal^\star(\lambdab),
\end{equation}
with the direct relationship (for any $\xb_t \in \Xcal_t, \lambdab \in \Lambda$):
\begin{equation}\label{eq:lcal_hatlcal}
\hat{\Lcal}_t(\xb_t,\lambdab) = \Lcal_t(\xb_t,\lambdab) - \lambdab^\top(\hAb_t - \Ab_t )\xb_t.
\end{equation}
Combining the above expression~(\ref{eq:lcal_hatlcal}) along with the decomposition detailed in~(\ref{eq:definition_gcal}), it can be shown (see Lemma A in the Appendix) that the residual term due to the estimation of $\Ab_t$ is given by
 \begin{equation}\label{eq:definition_scal_A}
\Scal_\Ab \defin \frac{1}{T} \sum_{t=1}^T  \hlb_t^\top(\hAb_t - \Ab_t )(\xb_t^\star- \hxb_t).
\end{equation} As a result of having to estimate $\mathbf{A}_t$, we present Algorithm \ref{alg:with_unknown_At}, where we also apply an OPSM to estimate the constraint matrices $\Ab_t$'s.
\begin{algorithm}[t]
   \caption{Online saddle point optimization with estimated constraint matrices $\hAb_t$}
   \label{alg:with_unknown_At}
\begin{algorithmic}
   \STATE {\bfseries Input:} Initial $\hAb_1$ and dual variable $\hlb_1$, step sizes $\eta_t, \nu_t$
   \FOR{$t=1$ {\bfseries to} $T$}
   \STATE Receive $\hat{\Lcal}_t$   \texttt{//1-lookahead and $\hAb_t$}
   \STATE Compute $\hxb_t \in \arg\!\max_{\xb_t \in \Xcal_t}  \hat{\Lcal}_t(\xb_t, \hlb_t)$
   \STATE Receive $\Ab_t$
   \STATE Update $\hlb_{t+1} = \Pi_{\Lambda}[ \hlb_{t} - \eta_t \nabla_\lambdab \Lcal_t(\hxb_t, \hlb_t) ]$
   \STATE \texttt{//$\Gb_t$ a subgradient of $\| \Ab_t - \Ab \|_\fro$}
   \STATE Update $\hAb_{t+1} = \Pi_{\Acal}[ \hAb_{t} - \nu_t \Gb_t ]$
   \ENDFOR
\end{algorithmic}
\end{algorithm} 

Lemma \ref{lem:regret_At} then quantifies the cumulative regret incurred as a result of having to use estimated constraint matrices $\hAb_t$ to produce $\hxb_t$, where the result stems from the analysis of~\cite{Hall2013}.
\begin{lemma}\label{lem:regret_At} 
Consider Algorithm~\ref{alg:with_unknown_At}. For OPSM on matrices $\hat{\mathbf{A}}_t$ with step size $\nu_t = R_A/\sqrt{t}$, we have 
$$
 \frac{1}{T} \sum_{t=1}^T \| \hAb_t - \Ab_t \|_\fro
 \leq \frac{3}{\sqrt{T}} 
 \bigg[
 R_A + \sum_{t=1}^T \|  \Ab_t - \Ab_{t+1} \|_\fro 
  \bigg],
$$ and the term $\Scal_\Ab$ from the general regret bound~(\ref{eq:informal_regret_bound}), as defined in~(\ref{eq:definition_scal_A}), is upper-bounded by:
$$
\Scal_\Ab
 \leq \frac{6 R_\lambda R_x}{\sqrt{T}} 
 \bigg[
 R_A + \sum_{t=1}^T \|  \Ab_t - \Ab_{t+1} \|_\fro 
\bigg].
$$
\end{lemma}
\noindent Having introduced the main three components of our regret bound, we next formally present our main results.
\subsection{Main results and discussion}
We start by stating our core theorem:
\begin{theorem}\label{thm:main}
Assume \textbf{(A1)}-\textbf{(A6)} hold, and let $\xb_1^\star,\dots,\xb_T^\star$ be a sequence of optimal primal solutions for (\ref{eq:primal_cost}). 
Define the optimal error vectors $\eb^\star_t \defin \Ab_t \xb_t^\star - \bb_t$ stacked into the vector $\underline{\eb}^\star \in \Real^{Tm}$. Let $M_\eb \defin \max_{t\in\Int{T-1}} \Psi_t(\underline{ \eb}^\star)$.
Algorithm~\ref{alg:with_unknown_At} generates a sequence $(\hxb_1,\dots,\hxb_T) \in \prod_{t=1}^T \Xcal_t $ satisfying
$$
\Pcal^\star  -  \Pcal(\hxb_1,\dots,\hxb_T) \leq  \Rcal_T + \Scal_{\eb} + \Scal_{\Ab}
$$
where
$
\Rcal_T = \frac{2R_\lambda G}{\sqrt{T}},
$ 
$
\Scal_e = M_\eb\frac{2 R_\lambda}{\sqrt{T}}
$
and
$
\Scal_\Ab = \frac{6 R_\lambda R_x}{\sqrt{T}} 
 [R_A + \sum_{t=1}^T \|  \Ab_t - \Ab_{t+1} \|_\fro] .
$
If we additionally assume that

~

\textbf{(A7)} $\Ecal$ has Lipschitz continuous gradients over its domain with parameter $L > 0$,

~

\noindent we can take instead $\Rcal_T = \frac{L G^2}{2T} \log(e T)$ and the term
$
M_\eb\frac{2 R_\lambda}{\sqrt{T}}
$
can be replaced by
$
M_\eb  \frac{L G \log(e T)}{T}.
$
\end{theorem}
The proof of the results can be found in the Appendix. The statement of Theorem~\ref{thm:main} calls for some comments. Omitting the contribution of terms depending on $\{R_\lambda, R_x, R_A, G\}$, Theorem~\ref{thm:main} guarantees that when $\Ab_t$ is assumed to be known, the regret  
$
\Pcal^\star  -  \Pcal(\hxb_1,\dots,\hxb_T)
$
is upper-bounded by terms scaling with $\Ocal(M_\eb/\sqrt{T})$, or $\Ocal(M_\eb\log(T)/T)$ respectively for the convex or strongly convex cases. As a result, the upper bound is mainly driven by how smooth the sequence $\{\eb^\star_t\}_{t=1}^T$ varies over time, as measured by $M_\eb$. In the more challenging setting where the constraint matrices $\Ab_t$'s are also estimated, we pay an additional cost
$\Ocal\big(\frac{1}{\sqrt{T}} [ 1 + \sum_{t=1}^T \|  \Ab_t - \Ab_{t+1} \|_\fro ]\big)$, so that the potential strong convexity of $\Ecal^*$ (or equivalently, the gradient Lipschitz continuity of $\Ecal$) plays only a secondary role compared to the leading term depending on the smoothness of the $\Ab_t$'s.
\par
\textbf{Possible instantiations of $\Ecal$:} We now present possible valid instantiations of the error function $\Ecal$. 
A first example for which the set of assumptions \textbf{(A1)-(A2)} can be satisfied considers $\Ecal(\zb) = r \cdot \Omega(\zb)$ where $\Omega$ refers to any norm on $\Real^m$. 
Indeed, it is well known (e.g., see Example 3.26 in \cite{Boyd2004}) that $\Ecal^\star$ corresponds in this case to the indicator function of the ball for the dual norm $\Omega^*$ with radius $r$. In particular, if $\Omega$ is the $\ell_2$ norm, we can take $R_\lambda = r$.
\par
Tighter regret bounds can be obtained if the function $\Ecal$ is taken to be gradient Lipschitz continuous (see Theorem~\ref{thm:main} and Assumption~\textbf{(A7)}), while preserving the boundedness of the domain $\Lambda$ of its Fenchel conjugate. A possible choice in this case is an instance of a Huber function (e.g., see Section 10.6 in~\protect\cite{Hastie2009}), as detailed in Table~\ref{tab:ecal_compact_version}.
\par
In the setting of the online ad allocation problem, it is sometimes required to impose an \textit{asymmetric} error function (e.g., to penalize under- and over-delivery differently). We show in Table~\ref{tab:ecal_compact_version} corresponding valid instantiations of $\Ecal$ that depend only on their arguments via their positive parts.
%\footnote{Note that the symmetry assumption~\textbf{(A8)} does not hold anymore, but it can be shown that we can still obtain an upper bound of $\Scal_\eb$ along the lines of Lemma~\ref{lem:upperbound_scal_e}, though with a less compact expression.}
%
\begin{table}
\centering
\hspace*{-0.3cm}\begin{tabular}{ccc}
$\Ecal(\mathbf{z})$ & $\Ecal^\star(\boldsymbol{\lambda})$ & $\Lambda$  \\ \hline
\\
 $ \Hcal_{R_\lambda,L}( \|\mathbf{z} \|_2) $ &  ${\cal I}_{\Bcal_{R_\lambda}}(\boldsymbol{\lambda}) + \frac{1}{2L} \| \boldsymbol{\lambda} \|_2^2$ & $\Bcal_{R_\lambda}$  \\
 $\Hcal_{R_\lambda,L}( \|[\zb]_+\|_2)$ & ${\cal I}_{\Bcal_{R_\lambda} \cap \Real^m_+}(\boldsymbol{\lambda}) + \frac{1}{2L}\|\boldsymbol{\lambda}\|_2^2$ & $\Bcal_{R_\lambda} \cap \Real^m_+$  \\
 $R_\lambda \cdot \|\mathbf{z} \|_2$ & ${\cal I}_{\Bcal_{R_\lambda}}(\boldsymbol{\lambda})$ & $\Bcal_{R_\lambda}$  \\
 $R_\lambda \cdot \| [\mathbf{z}]_+ \|_2$ & ${\cal I}_{\Bcal_{R_\lambda} \cap \Real^m_+}(\lambdab)$ & $\Bcal_{R_\lambda} \cap \Real^m_+$ 
\end{tabular}
\caption{Examples of non-additive penalty functions.
For simplicity, we display only expressions based on the $\ell_2$ norm, but similar formula can be obtained for $\ell_q, q \in\{1,\infty\}$.
We recall that 
$\Lambda \defin \dom(\Ecal^\star) = \{\lambdab : \Ecal^\star(\lambdab) < +\infty\}$. We define $\Bcal_r$ as the Euclidean ball in $\Real^m$ with radius $r$, while $\Ical_C$ is the indicator function of the convex set $C$. 
The function $\Hcal_{r,s}$ is defined as $\Hcal_{r,s}(t) \defin \frac{1}{2} \min\{s t^2,\frac{r^2}{s}\} + r [ |t| -\frac{r}{s} ]\protect_+$.
The expressions are derived in Lemma~B in the Appendix.\label{tab:ecal_compact_version}}
\end{table}
\par
\textbf{Complexity of computing $\hxb_t$:} Assumption~\textbf{(A4)}, about the exact computational oracle for $\hxb_t$, hides, and concentrates, the difficulties related to the optimization with respect to the primal variables $\xb_t$.
We can observe, in the light of the decomposition~(\ref{eq:definition_gcal}), that an approximate maximization---for instance leveraging the concavity of $f_t$---would lead to residual terms that would be hard to control in our online setting. Although it may appear at first sight that an exact maximization is an overly strong requirement, wide classes of problems fall within the scope of this assumption; we can for instance cite the exactly-solved subproblems that are commonly encountered in the context of proximal methods~\cite{Parikh2013} or conditional-gradient algorithms~\citep[see Table 1]{Jaggi2013}.
We conjecture that, thanks to Assumption~\textbf{(A4)}, we may weaken Assumption~\textbf{(A3)} by, for instance, trying to drop the concavity assumption of the $f_t$'s and the convexity of the $\Xcal_t$'s.
\par
\textbf{Lower bounds:}
Based on Assumptions \textbf{(A1)-(A2)-(A3)}, we know that the minimax equality for our saddle point problem holds~\cite{Sion1957}. In particular, defining the (convex) dual function
$$
\Dcal(\lambdab) \defin \frac{1}{T} \sum_{t=1}^T \Dcal_t(\lambdab) \defin \frac{1}{T} \sum_{t=1}^T \max_{\xb_t \in \Xcal_t} \Lcal_t(\xb_t,\lambdab),
$$
we have the equality
$$
\Pcal^\star  -  \Pcal(\hxb_1,\dots,\hxb_T)
=
\min_{\lambdab \in \Lambda}\Dcal(\lambdab) - \Pcal(\hxb_1,\dots,\hxb_T).
$$
This new relationship makes it possible to derive complementary guarantees for our problem. For instance, when we need not estimate the constraint matrices $\Ab_t$'s,\footnote{When $\Ab_t$'s are also estimated, the lower bound is more involved, and requires the introduction of an estimated dual function $\hat{\Dcal}_t$ along the lines of~(\ref{eq:hatlcal}).} it can be shown that we can \textit{lower bound} 
$
\Pcal^\star  -  \Pcal(\hxb_1,\dots,\hxb_T)
$
by
$$
\bigg[  \frac{1}{T}  \sum_{t=1}^T \Dcal_t(\hlb_t) - \Pcal(\hxb_1,\dots,\hxb_T) - \Rcal_T  \bigg]_+,
$$
which sets a lower limit on the best performance we could get with the sequences $\{\hxb_t,\hlb_t\}_{t=1}^T$. Interestingly, this lower bound can not only be practically computed, but it also gives an indication about the inherent difficulty of the problem at hand (e.g., if it is large, the online strategy cannot compete efficiently with its offline counterpart).
\section{Experiments}

In this section, we conduct two sets of experiments. First, we consider a synthetic data and demonstrate the benefit of handling truly non-additive penalties over a heuristic based on an additive relaxation. Second, we report results on an online advertising data set consisting of a sample of 3.3 million bid requests gathered from a large ad serving system used in production at Amazon. Hence, these data account for external advertiser constraints encountered in practice, such as user behavioral targeting.  
We show that the rate at which regret vanishes matches our theory.

\subsection{Additive versus non-additive modeling}

We first investigate if the use of a non-additive penalty on average constraint violations yields improvements in reward and/or constraint violations compared to the additive error formulation adopted in standard online learning. The additive long-term penalty is defined as follows:
\begin{equation}\label{eq:additive_error}
\frac{1}{T}\sum_{t=1}^T\Big[ f_t(\xb_t)  - \Ecal(\Ab_t \xb_t - \bb_t)\Big] .
\end{equation}

In the experiments, we consider a linear reward function $f_t(\xb) = \ub_t^\top\xb_t$, with the simplicial constraint $\xb_t\in \Xcal_t \defin \{  \xb_t \in \Real_+^d : \oneb^\top \xb_t \leq 1 \}$. These choices of $f_t$ and $\Xcal_t$ are reminiscent of the online ad allocation setting we explore at greater length in Section~\ref{sec:online_ad_allocation_exp}. We simulated data with $m=25, d= 10,T=200$ and where $\{\mathbf{A}_t, \mathbf{b}_t, \mathbf{u}_t\}$ were generated to be standard random Gaussian matrices and vectors normalized to unit norms.  We ran our online algorithm with primal-dual updates dictated by the non-additive problem formulation in (\ref{eq:primal_cost}) and compared it to a baseline algorithm with primal-only updates following a standard online learning formulation with additive penalties on constraint violations as in~(\ref{eq:additive_error}). Since the exact maximization $\max_{\xb_t\in\Xcal_t} [f_t(\xb_t)  - \Ecal(\Ab_t \xb_t - \bb_t)]$ 
may not be obtained in closed-form in this case, we used the solver \texttt{CVXPY}~\cite{Diamond2014}. We shall refer to the two approaches as respectively \texttt{Non-additive} and \texttt{Additive}.
\par
So as to cover different types of geometries and convexity assumptions, we consider several choices of penalty functions, namely $\Ecal(\zb) = R_\lambda \cdot \| \zb \|_q$ for $q \in \{1,2,\infty\}$, along with $\Ecal(\zb) = R_\lambda \cdot \Hcal_{1,1}(\|\zb\|_2)$ where $\Hcal_{1,1}$ is the Huber function defined in Table~\ref{tab:ecal_compact_version}. 
For each 
$\Ecal$, and for, both, \texttt{Non-additive} and \texttt{Additive}, we compute the reward over $T$ rounds given by $\frac{1}{T}\sum_t \mathbf{u}_t^\top\hxb_t$ and the (normalized) non-additive penalty 
$
\Ecal(\frac{1}{T}\sum_t\mathbf{A}_t\hat{\mathbf{x}}_t-\mathbf{b}_t)/R_\lambda
$ 
for the primal variables $\hat{\mathbf{x}}_t$ generated online by the two algorithms.
\par
Figure~\ref{fig:reward_vs_ecal} shows the resulting reward versus constraint violation for varying values of $R_\lambda$, averaged over 10 generations of $\{\mathbf{A}_t, \mathbf{b}_t, \mathbf{u}_t\}$. As can be seen from the curves, in applications where constraint violation are to be measured as a non-additive penalty, the additive relaxation~(\ref{eq:additive_error}) leads to significant deterioration of the performance compared to the non-additive penalty. In particular, the domain of the achievable constraint violation is very narrow when applying the additive heuristic, meaning that there is little room for making trading-offs. As expected, we also observe that points of \texttt{Non-additive} and \texttt{Additive} superimpose in the regime where $R_\lambda \ll 1$, i.e., when the two formulations focus on the optimization of the reward. 
Finally, we remark that since the quality of the additive relaxation~(\ref{eq:additive_error}) essentially hinges on the gap in the Jensen's inequality $\Ecal(\frac{1}{T} \sum_{t=1}^T \Ab_t \xb_t - \bb_t) \leq \frac{1}{T} \sum_{t=1}^T \Ecal( \Ab_t \xb_t - \bb_t)$, and therefore on the distribution of the residuals $\{\Ab_t \xb_t - \bb_t\}_{t=1}^T$, we provide in the Appendix additional simulations where $\{\Ab_t,\bb_t,\ub_t\}$ are generated according to different distributions, viz, Cauchy, uniform and gamma. In a nutshell, the same conclusions hold, even for the distributions that appear to make the relaxation~(\ref{eq:additive_error}) tighter. 
\begin{figure}[t]
  \centering
  \includegraphics[width=0.5\textwidth]{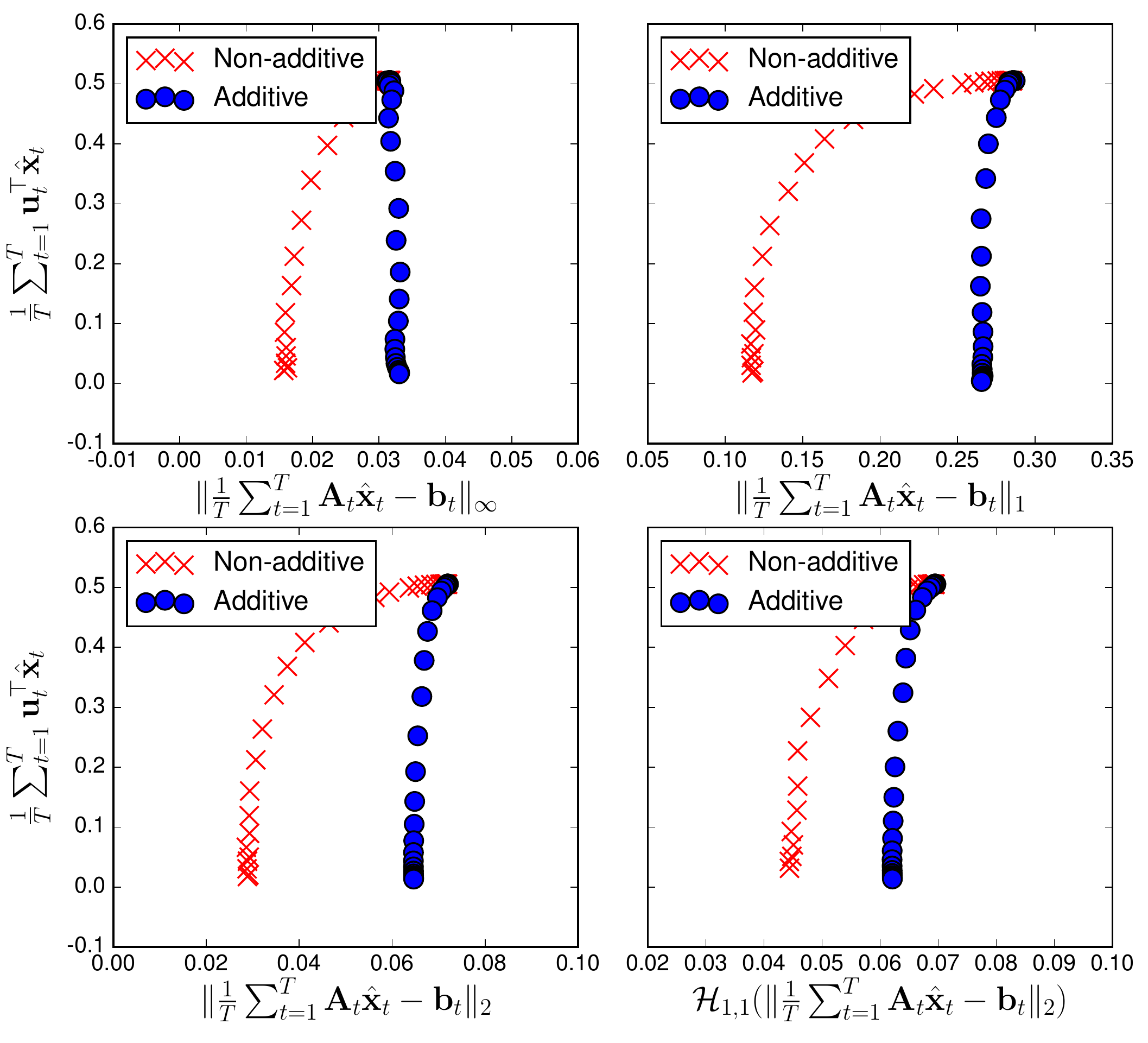}
  \caption{Reward $\frac{1}{T}\sum_t \ub_t^\top\hxb_t$ as a function of non-additive penalty $\Ecal(\frac{1}{T}\sum_t\mathbf{A}_t\hat{\mathbf{x}}_t-\mathbf{b}_t)/R_\lambda$ for 
  $R_\lambda = 2^{\gamma}$ with $\gamma \in \{-8,-7.5,\dots, 10\}$.
  Red crosses correspond to our proposed online algorithm, while blue circles stand for a baseline algorithm with additive penalties.
  Each subplot displays a different instantiation of $\Ecal$, namely $\Ecal(\zb) = R_\lambda \cdot \| \zb \|_q$ for $q \in \{1,2,\infty\}$ and $\Ecal(\zb) = R_\lambda \cdot \Hcal_{1,1}(\|\zb\|_2)$ where $\Hcal_{1,1}$ is defined in Table~\ref{tab:ecal_compact_version} (best seen in color).}\label{fig:reward_vs_ecal}
\end{figure}
\subsection{Regret convergence on real online advertising data}\label{sec:online_ad_allocation_exp}

In the second experiment, we focus on the online ad allocation problem in display advertising subject to long-term constraints on budget consumed per ad. More formally, in the display advertising setting, each user visit to a website or app triggers a bid request $i$, where each bid request has some subset of the $m$ possible ads that can be served. Indexing ads by $j \in\Int{m}$, the welfare of serving an impression for ad $j$ (i.e., the value of showing the ad once to a user) for bid request $i$ is given by $u_{ij}\geq 0$. For bid request $i$, let primal variable $x_{ij} = 1$ correspond to our decision to allocate ad $j$ an impression, with $x_{ij} = 0$ otherwise. 
\par
To map this problem to our formulation, we will further partition the time dimension into disjoint intervals $t \in\Int{T}$ where the number of bid requests in each interval is equal to $N$.
We let $I_t$ be the set of bid requests in round $t$ with $|I_t| = N$. Thus, for a given $t$, the primal variable $\xb_t$ will denote the flattened matrix $x_{ij}$ for $i \in I_t$ and for all ads $j  \in\Int{m}$, so that $\xb_t \in [0,1]^d$ with $d = N \cdot m$. Since we can show at most one ad per bid request, we additionally have the simplicial constraints $\sum_{j=1}^m x_{ij} \leq 1$ for each bid request $i$, and we therefore define the set $\Xcal_t \subseteq [0,1]^d$ as the Cartesian product of those simplices for each $i \in I_t$. We define the welfare vector $\ub_t \in \Real^d$ following the same flattening operation, so that $f_t(\mathbf{x}_t) = \mathbf{u}_t^\top\mathbf{x}_t$. Matrix $\Ab_t \in \Real^{m \times d}$ is such that the $j$-th entry $[\Ab_t \xb_t]_j$ represents the amount of budget consumed for ad $j$  as a result of the allocation vector $\xb_t$ under a cost per impression model.
Finally, the non-additive long-term constraints
are modeled as $\frac{1}{T} \sum_t \mathbf{A}_t\mathbf{x}_t - \mathbf{b}_t$.  Vector $\mathbf{b}_t$ is equal to the constant vector $\bb$ of ad budgets. Hence, the penalty on $\frac{1}{T} \sum_t \mathbf{A}_t\mathbf{x}_t - \mathbf{b}_t$ penalizes deviations from spending 100\% of each ad's budget. We stress the fact that, although $\bb$ is a constant vector, we are in a non-stationary regime since $\Ab_t$ is not equal to some constant matrix $\Ab$ independent of $t$. 

The goal of the experiment is to compare the regret behavior of our proposed online algorithm for different choices of non-additive $\Ecal$ on the online advertising data set from live traffic. For each bid request in our data set, each eligible ad candidate comes with a pre-defined welfare and cost to the advertiser, which determine $\mathbf{u}_t$ and $\mathbf{A}_t$ respectively. 
We partitioned the 3.3 million impressions into batches of equal size, for a total of $T = 10000$ rounds. We computed the average cumulative regret $\Pcal^\star - \Pcal(\hat{\mathbf{x}}_1,\cdots,\hat{\mathbf{x}}_T)$ as a function of $T$ for (symmetric) $\Ecal$ chosen to be 
i) $\Ecal(\zb) = R_\lambda \cdot \|\zb\|_1$ referred to as \texttt{Convex} and ii) $\Ecal(\zb) =  \Hcal_{R_\lambda,1}(\|\zb\|_2)$ referred to as \texttt{Strongly convex}. For all problem instantiations, we set $R_\lambda = 50000$. We computed the term $\Pcal(\hat{\mathbf{x}}_1,\cdots,\hat{\mathbf{x}}_T)$ by running our online primal-dual algorithm over 20 permutations of bid requests. The term $\Pcal^\star$ was obtained using an offline primal-dual method where the dual variable updates correspond to an offline subgradient (resp. gradient) method for the convex  (resp. strongly convex) $\Ecal^*$. Figure \ref{fig:cum_regret} shows the resulting cumulative regret as a function of $T$, averaged over the 20 runs (with error bars that are negligible). We observe that the cumulative regrets incurred for, both, \texttt{Convex} and \texttt{Strongly convex} decrease as $T$ increases, with the latter decreasing at faster rate than the former as predicted by the theory.
\begin{figure}[t]
    \centering
      \includegraphics[width=0.5\textwidth]{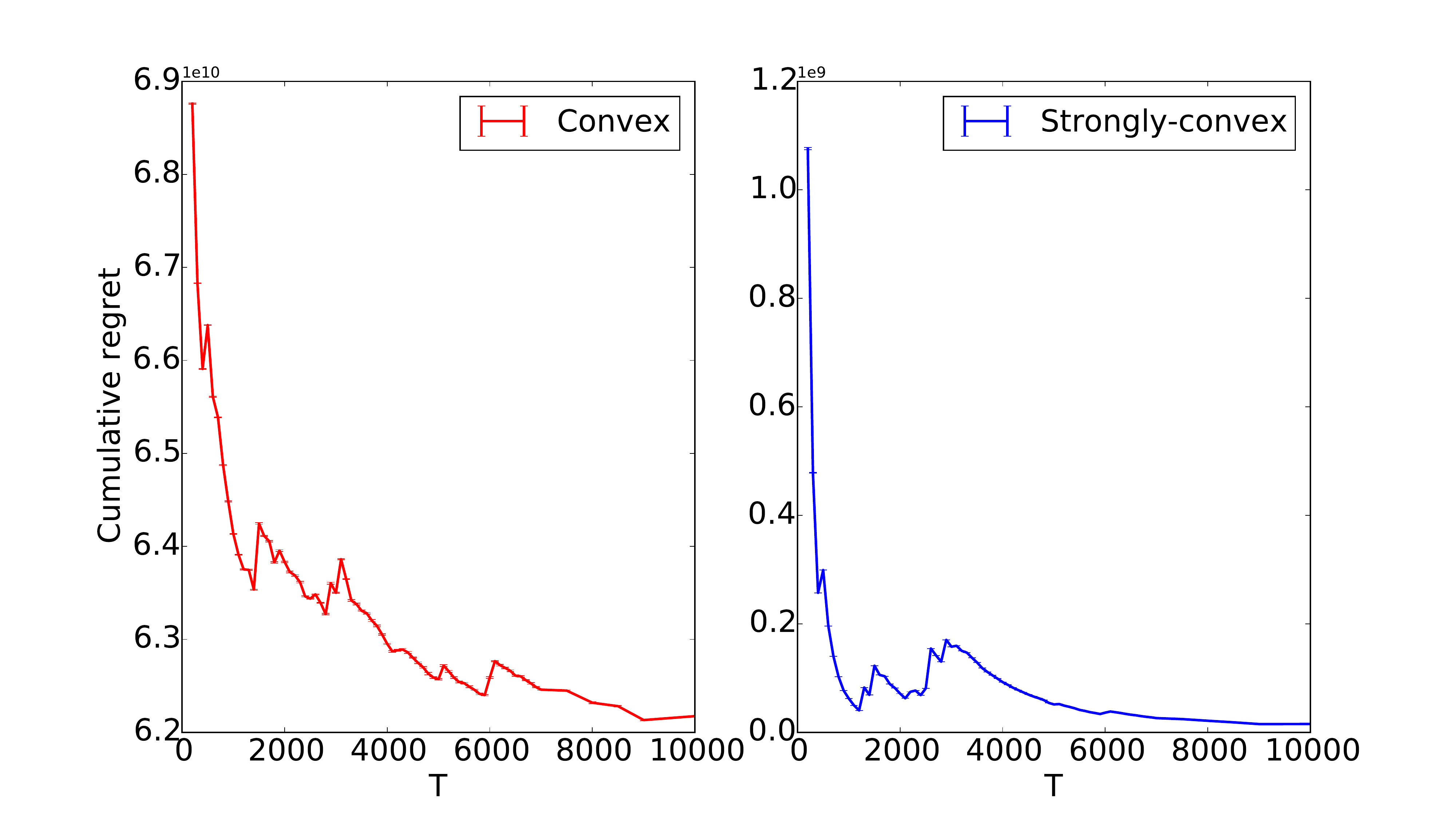}
\caption{Cumulative regret, averaged over 20 runs, measured on the data collected by a real ad serving system as a function of the number of rounds $T$ for $\Ecal(\mathbf{z}) = R_\lambda\|\mathbf{z}\|_1$ (red) and $\Ecal(\mathbf{z}) = {\cal H}_{R_\lambda,1}(\|\mathbf{z}\|_2)$ (blue). Figure best seen in color.}\label{fig:cum_regret}
\end{figure}

\section{Appendix: Proofs}
\renewcommand*{\thelemma}{\Alph{lemma}}
\setcounter{lemma}{0}
\subsection{Proof of Lemma 1}
\begin{proof}
The two claims of the lemma can be found respectively in \cite{Zinkevich2003,Hazan2007} and follows from (possibly, strong) convexity of ${\cal L}_t(\hxb_t,\boldsymbol{\lambda})$ with respect to $\lambdab$.
For completeness, we briefly repeat the core arguments from \cite{Zinkevich2003,Hazan2007}.  The conclusions come from the fact that the term 
$$
\sum_{t=1}^T \Lcal_t(\hxb_t,\hlb_t)  - \Lcal_t(\hxb_t,\lambdab)
$$ 
is shown to be upper bounded by, with $a_t \defin \| \lambdab - \hlb_t \|_2^2$,
$$
\frac{1}{2} \sum_{t=1}^T \bigg[ \frac{1}{\eta_t} \big( a_t -   a_{t+1} \big) - \kappa a_t + \eta_t G^2 \bigg]\\
$$
in turn upper bounded by
$$
\frac{a_1}{2} \bigg[ \frac{1}{\eta_1} - \kappa  \bigg]
+
 \frac{G^2}{2} \sum_{t=1}^T \eta_t
+
\frac{1}{2} \sum_{t=2}^T a_t  \bigg[ \frac{1}{\eta_t} -  \frac{1}{\eta_{t-1}} -  \kappa \bigg].
$$
The choices of the step sizes lead to the advertised  instantiations of $\Rcal_T$.
\end{proof}
\subsection{Proof of Lemma 2}
\begin{proof}
The result stems from an application of strong duality for the above convex program. First, notice that
$$
\frac{1}{T}  \sum_{t=1}^T
 \Gcal_t(\lambdab_t) = 
\frac{1}{T} \sum_{t=1}^T \lambdab_t^\top \eb^\star_t - \Ecal^*(\lambdab_t)
-\Ecal \Big(\!
\frac{1}{T}  \sum_{t=1}^T \eb^\star_t 
\!\Big) .
$$
Introduce the equality constraint $\ab = (  \Deltab \otimes \Ib ) \underline{\lambdab}$---we remind that $\underline{\lambdab} \in \Real^{T\cdot m}$ stands for the vector formed by stacking the $m$-dimensional $\lambdab_t$'s---and write the corresponding Lagrangian (we momentarily omit the constant terms independent of $\lambdab_t$ that have no effects on the maximization), that is, for any $\lambdab_t \in \Lambda, \ab \in \Real^{(T-1)\cdot m}, \alphab \in \Real^{(T-1)\cdot m}$ and $\beta \geq 0$:
\begin{align}
L(\underline{\lambdab},\ab,\alphab, \beta) &= \frac{1}{T}  \sum_{t=1}^T \lambdab_t^\top \eb^\star_t - \Ecal^*(\lambdab_t) \nonumber \\ &+ \alphab^\top (\ab - ( \Deltab \otimes \Ib) \underline{\lambdab}) + \beta (\varepsilon - \Omega_{1/2}(\ab)).\nonumber
\end{align} Using the fact that $\Ecal^*$ is the Fenchel conjugate of $\Ecal$ and that the conjugate of the $\ell_1/\ell_2$ norm $\Omega_{1/2}$ is the indicator function of the unit ball for its dual $\ell_\infty/\ell_2$ norm (Example 3.26 in \cite{Boyd2004}), we can maximize out the primal variables:
$$
\max_{\underline{\lambdab},\ab} L(\underline{\lambdab},\ab,\alphab, \beta) =
 \frac{1}{T}  \sum_{t=1}^T 
\Ecal \Big( 
\eb^\star_t - \big[  T (  \Deltab \otimes \Ib )^\top \alphab  \big]_t
\Big)
\!
+ 
\!
\beta \varepsilon
$$
under the constraint $\max_{t=1\dots T-1} \|  \alphab_t  \|_2 \leq \beta$.
Minimizing out $\beta$, and making the change of variable $\alphab \leftarrow T \cdot \alphab$ leads to the claimed dual minimization problem. Finally, the initial claim holds with equality, since Slater's constraint qualification applies here to guarantee strong duality (Section 5.2.3 in \cite{Boyd2004}).
\end{proof}
\subsection{Proof of Lemma \ref{lem:upperbound_scal_e}}
\begin{proof}
The equality in Lemma~\ref{lem:scal_e_as_min} has a minimization over $\bal$ on its right-hand-side. Therefore,  we have that for every $\bal \in \R^{(T-1)\cdot m}$ we have $\cS_\be$ upper bounded by $\cS_\be(\bal)$. Specifically, let $\bal' \eqdef [((\bDel \bDel^\top)^{-1} \bDel) \otimes \bI] \underline{\be}^*$ and we have that $\cS_\be  \leq \cS_\be(\alpha')$. First, we compute the term inside the first sum of $\cS_\be(\bal')$ for a fixed $t \in \{1, \dots, T\}$ 
\begin{align*}
\be_t^* -  \left[ (\bDel \otimes \bI)^\top[((\bDel \bDel^\top)^{-1} \bDel) \otimes \bI] \underline{\be}^* \right]_t &= \left[ \underline{\be}^* -   (\bDel \otimes \bI)^\top[((\bDel \bDel^\top)^{-1} \bDel) \otimes \bI] \underline{\be}^* \right]_t \\
&= \left[ \underline{\be}^* -   ([(\bDel^\top(\bDel \bDel^\top)^{-1} \bDel) \otimes \bI] \underline{\be}^* \right]_t \\
&= \left[[(\bI - \bDel^\top(\bDel \bDel^\top)^{-1} \bDel) \otimes \bI] \underline{\be}^* \right]_t
\end{align*}
The matrix $\bDel \in \R^{(T-1) \times T}$ is defined by \[\bDel_{ij} = \begin{cases}
1 \quad& i=j, \\
-1 \quad& i = j-1, \\
0 \quad&\mbox{otherwise}.
\end{cases}\]
By direct computation one can show, that $\bDel \bDel^\top \in \R^{(T-1)\times (T-1)}$ is the tridiagonal matrix \[(\bDel \bDel^\top)_{ij} = \begin{cases}
2 \quad&i=j, \\
-1 \quad&|i-j| = 1, \\
0 \quad& \mbox{otherwise}.
\end{cases}\]
There exists a close form solution for the inversion of an arbitrary tridiagonal matrix \cite{TridiagonalInversion}. Using this we get that the matrix $(\bDel \bDel^\top)^{-1} \in \R^{(T-1) \times (T-1)}$ is defined by
\[ [(\bDel \bDel^\top)^{-1}]_{ij} = \begin{cases}
\frac{i(T-j)}{T} \quad&i < j, \\
\frac{(T-i)j}{T} \quad&i \geq j.
\end{cases}\]
No,w we can again proceed with a direct calculation to compute $(\bDel \bDel^\top)^{-1} \bDel \in \R^{(T-1) \times T}$ and get
\[[(\bDel \bDel^\top)^{-1} \bDel]_{ij} = \begin{cases}
-\frac{i}{T} \quad& i < j,\\
\frac{T-i}{T} \quad& i \geq j.
\end{cases}\]
Finally, computing the full product $\bDel^\top (\bDel \bDel^\top)^{-1} \bDel \in \R^{T \times T}$ we get
\[[\bDel^\top(\bDel \bDel)^{t-1} \bDel]_{ij} = \begin{cases}
\frac{T-1}{T} \quad& i=j, \\
-\frac{1}{T} \quad& i\neq j.
\end{cases}\]
Observe, that the matrix $\bI - \bDel (\bDel \bDel^\top)^{-1} \bDel \in \R^{T \times T}$ is a matrix full of values $1/T$. It follows that \[\be_t^* -  \left[ (\bDel \otimes \bI)^\top[((\bDel \bDel^\top)^{-1} \bDel) \otimes \bI] \underline{\be}^* \right]_t = \frac{1}{T}\sum_{i=1}^T \be_i^*, \qquad \forall t \in \{1, \dots, T\}\]
Also, observe that the first two terms in $\cS_\be(\bal')$ cancel out and we have \[\cS_\be(\bal') = \frac{\epsilon}{T} \Omega_{\infty/2}(\bal') =\frac{\epsilon}{T} \Omega_{\infty/2}\left( [((\bDel \bDel^\top)^{-1} \bDel) \otimes \bI] \underline{\be}^*  \right)\]
Plugging in the matrix $(\bDel \bDel)^{-1} \bDel$ computed earlier in the above expression we get the final result.
\end{proof}
\subsection{Proof of Lemma \ref{lem:concentration_bound}}
Let $\bW_t \in \R^{m \times mT}$ be a block matrix defined by
\begin{equation} \label{def:Wt}
\bW_t \eqdef \left[  ~ \underbrace{\left(\frac{T - t}{T}\right) \bI ~\bigg|~ \dots ~\bigg|~ \left(\frac{T - t}{T}\right) \bI}_{t} ~\bigg|~ \underbrace{\left(-\frac{t}{T}\right)\bI ~\bigg|~ \dots ~\bigg|~ \left(-\frac{t}{T}\right) \bI}_{(T-t)}  ~ \right]
\end{equation}
Observe that $Z_t = \|\bW_t \underline{\be}^*\|_2$.  From now on $t$ will be assumed to belong to $\{1, \dots, T-1\}$ in all derivations. Let $s > 0$. Then using the convexity and monotonicity of the exponential together with Jensen's inequality, the fact that sum of elements is more than the maximal element, linearity of expectation, and the fact that the sum can be bounded by the number of terms times the maximal element, we get the following sequence of inequalitites
\begin{equation} \label{eq:bounding_exp_max}
e^{s \Exp \left[\max_t Z_t^2\right]} \leq \Exp \left[e^{s \max_t Z_t^2}\right] = \Exp \left[\max_t e^{sZ_t^2}\right] \leq  \Exp \left[ {\textstyle\sum_{t}} e^{sZ_t^2} \right] = {\textstyle \sum_t} \Exp \left[e^{sZ_t^2}\right] \leq T \max_t \Exp\left[ e^{sZ_t^2} \right]
\end{equation}
Assume $0 \leq s < 1/(2\sigma^2 \|\bW_t^\top \bW_t\|_2)$. Then according to Remark 2.3 in \cite{SubgaussianQuadratic} one has the bound \begin{equation} \label{eq:subgaussian_quadratic}
\Exp [ \exp(s \|\bW_t \underline{\be}^*\|^2)] \leq \exp \left( \sigma^2 \Tr\left(\bW_t^\top \bW_t\right)s + \frac{\sigma^4 \Tr\left( (\bW_t^\top \bW_t)^2 \right) s^2 + \|\bW_t\bmu\|^2 s}{1 - 2 \sigma^2 \|\bW_t^\top \bW_t\|_2 s} \right)
\end{equation}
Observe that the right-hand-side of \eqref{eq:subgaussian_quadratic} is increasing in the arguments $\Tr(\bW_t^\top t)$, $\Tr((\bW_t^\top \bW_t)^2)$,  $\|\bW_t \bmu\|^2$, and $\|\bW_t^\top \bW_t\|_2$. From the definition of $\bW_t$ in \eqref{def:Wt} and a maximization in $t$ we have \begin{equation*} \label{eq:Frob_W_t}
\|\bW_t\|^2_F \stackrel{\eqref{def:Wt}}{=} \frac{m t(T-t)}{T} \leq \frac{mT}{4}
\end{equation*}
which can be further used to bound
\begin{align*}
\Tr(\bW_t^\top \bW_t) &= \|\bW_t\|_F^2 \leq \frac{mT}{4} \\ 
\Tr((\bW_t^\top \bW_t)^2) &= \|\bW_t^\top \bW_t\|_F^2 \leq \|\bW_t\|_F^4 \leq \frac{m^2T^2}{16} \\
\|\bW_t^\top \bW_t\|_2 &\leq \|\bW_t^\top \bW_t\|_F \leq \|\bW_t\|_F^2 \leq \frac{mT}{4}
\end{align*}
using standard arguments. Also, from definition \eqref{def:Zt} we have $\|\bW_t \bmu\|^2 \leq \omega_{\bmu}$. Combining these bounds with \eqref{eq:subgaussian_quadratic} and setting $s \eqdef s^*$ defined by \begin{equation} \label{def:s_star}
0 < \left[ s^* \eqdef \frac{4\sqrt{\log(T)}}{\sqrt{\sigma^2 m T}\left(\sqrt{\sigma^2 m T + 8 \omega_{\bmu}} + 2\sqrt{\sigma^2 mT \log(T)}\right) } \right] < \frac{2}{\sigma^2 m T} \leq \frac{1}{2 \sigma^2 \|\bW_t^\top \bW_t\|_2}
\end{equation}
we get \begin{align} 
\Exp[\exp(s^* \|\bW_t \underline{\be}^*\|^2)] &\leq \exp\left( \frac{ s^*}{4} \left( \sigma^2mT + \frac{\sigma^4 m^2 T^2 s^* + 16 \omega_{\bmu} }{4 - 2 \sigma^2 m T s^*} \right) \right) \nonumber \\ &\stackrel{\eqref{def:s_star}}{=}  \exp \left( \frac{s^*}{4}  \left(\sigma^2 m T + 4 \omega_{\bmu} +  \sqrt{\sigma^2 m T + 8 \omega_{\bmu}} \sqrt{\sigma^2 m T\log(T)}\right)\right). \label{eq:after_s}
\end{align}
We finally have all the ingredients. Using the definition of $Z_t$ in \eqref{def:Zt}, using the non-negativity of $Z_t$, the concavity of the square root together with Jensen's inequality, the definition of the exponential, the inequality \eqref{eq:bounding_exp_max}, the inequality \eqref{eq:after_s} with the observation that the right-hand-side does not depend on $t$, the definition of $s^*$ in \eqref{def:s_star}, the bound $\sqrt{s} \leq \frac{1}{2}(s + 1)$, the fact that all the entries are positive, and $a + b = \sqrt{a^2 + 2ab + b^2}$ we get
\begin{align*}
\Exp \left[ \max_{t} Z_t\right] &= \Exp\left[\sqrt{ \max_t Z_t^2}  \right] \\ 
&\leq \sqrt{\Exp [\max_{t} Z_t^2]} \\
&= \sqrt{ \log \left( e^{ s^* \Exp \left[  \max_t Z_t^2    \right]}\right)/s^*} \\
&\stackrel{\eqref{eq:bounding_exp_max}}{\leq} \sqrt{ \log \left( T \max_t \Exp \left[ e^{s^* Z_t^2} \right]  \right) / s^*} \\
&\stackrel{\eqref{eq:after_s}}{\leq} \sqrt{\frac{\log(T)}{s^*} + \frac{1}{4}  \left(\sigma^2 m T + 4 \omega_{\bmu} +  \sqrt{\sigma^2 m T + 8 \omega_{\bmu}} \sqrt{\sigma^2 m T\log(T)} \right)  } \\
&\stackrel{\eqref{def:s_star}}{=}\frac{1}{2}\sqrt{\sigma^2 m T + 4 \omega_{\bmu} +  2\sqrt{\sigma^2 m T + 8 \omega_{\bmu}} \sqrt{\sigma^2 m T\log(T)} + 2 \sigma^2 m T \sqrt{\log(T)}} \\
&\leq \frac{1}{2}\sqrt{2\sigma^2 m T + 4 \omega_{\bmu} +  2\sqrt{\sigma^2 m T + 8 \omega_{\bmu}} \sqrt{\sigma^2 m T\log(T)} + \sigma^2 m T \log(T)} \\
&\leq \sqrt{\sigma^2 m T + \omega_{\bmu} +  2\sqrt{\sigma^2 m T + \omega_{\bmu}} \sqrt{\sigma^2 m T\log(T)} + \sigma^2 m T \log(T)} \\
&\leq \sqrt{\sigma^2 m T + \omega_{\bmu}} + \sqrt{\sigma^2 m T \log(T)}
\end{align*} 
which concludes the proof.

\subsection{Proof of Lemma \ref{lem:regret_At}}
\begin{proof}
We first apply the Cauchy-Schwartz inequality to get
\begin{align}
|\hlb_t^\top(\hAb_t - \Ab_t )(\xb_t - \hxb_t)| &\leq \|\hlb_t\|_2 \|\xb_t - \hxb_t\|_2 \triple \hAb_t - \Ab_t\triple_2 \nonumber \\ &\leq 2R_\lambda R_x \|\hAb_t - \Ab_t\|_\fro\nonumber
\end{align}
where we have used that the operator norm is smaller than the Frobenius norm.
We then leverage Theorem 4 from~\cite{Hall2013}: in their notation, we have $G_\ell \leq 1$ (since the sub-gradients of $\Ab \mapsto \| \Ab_t - \Ab \|_\fro$ are bounded by 1), $\sigma=1$, $M \leq R_A/2$ and $D_{\max} \leq 2 R_A^2$ along with $\Phi = \Ib$ and the learning rate $\eta_t = R_A/\sqrt{t}$. 
This leads to
\begin{align}
\sum_{t=1}^T \|  \hAb_t - \Ab_t  \|_\fro &- \min_{\tilde{\Ab}_1,\dots,\tilde{\Ab}_T \in \Acal} \sum_{t=1}^T \|  \tilde{\Ab}_t - \Ab_t  \|_\fro \nonumber \\
&\leq \frac{2R_A^2}{\eta_{T+1}} + \frac{2R_A}{\eta_{T}}V + \frac{1}{2} \sum_{t=1}^T \eta_t\nonumber\\ 
&\leq  \frac{2R_A}{\eta_{T}} (R_A + V) + R_A \sqrt{T}\nonumber\\
&\leq 3 \sqrt{T} (V + R_A),\nonumber
\end{align}
where $V$ stands for $\sum_{t=1}^T \|  \Ab_t - \Ab_{t+1} \|_\fro$.
\end{proof}
\subsection{Proof of Theorem 1}
\begin{proof}
The proof consists in putting together the components introduced in the Section 3 of the core paper.
We start from
$$
\Pcal^\star - \Pcal(\hxb_1,\dots,\hxb_T) = \frac{1}{T} \sum_{t=1}^T
\Lcal_t(\xb_t^\star,\lambdab^\star) - \Lcal_t(\hxb_t,\hlb)
$$
which, using the decomposition~(6), % has to be (\ref{eq:definition_gcal})
is equal to
$$
 \frac{1}{T} \sum_{t=1}^T
(\Lcal_t(\xb_t^\star,\lambdab^\star) - \Lcal_t(\xb_t^\star,\hlb_t)) + (\Lcal_t(\xb_t^\star,\hlb_t) -  \Lcal_t(\hxb_t,\hlb)).
$$
Moreover, Lemma A (see below) gives  
\begin{align}
\Lcal_t(\xb_t^\star,\hlb_t) - \Lcal_t(\hxb_t,\hlb) &\leq \Lcal_t(\hxb_t,\hlb_t)  - \Lcal_t(\hxb_t,\hlb) \nonumber \\ &+ \hlb_t^\top(\hAb_t - \Ab_t )(\xb_t ^\star- \hxb_t)\nonumber.
\end{align}
Recalling the definitions of $\Rcal_T$, $\Gcal_t$ and $\Scal_\Ab$ respectively in 
Lemma~\ref{lem:standard_static_regret}, 
(\ref{eq:definition_gcal}) 
and~(\ref{eq:definition_scal_A}),
we obtain
$$
\Pcal^\star - \Pcal(\hxb_1,\dots,\hxb_T) \leq 
\frac{1}{T} \sum_{t=1}^T \Gcal_t(\hlb_t)
+ \Rcal_T + \Scal_\Ab.
$$
Noticing that with the choice $\varepsilon/T \defin G \sum_{t=1}^T \eta_t/T$, the relationship~(\ref{eq:definition_scal_e})
holds, we finally obtain
$$
\Pcal^\star - \Pcal(\hxb_1,\dots,\hxb_T) \leq 
\Rcal_T + \Scal_\eb + \Scal_\Ab.
$$
\par
The rest of the proof follows by instantiating $\Rcal_T$, $\Scal_\eb$ and $\Scal_\Ab$. The value of $\Rcal_T$ is given in Lemma~1, %\ref{lem:standard_static_regret}
while the upper bound for $\Scal_\eb$ is described in Lemma~3. %\ref{lem:upperbound_scal_e}
\par
Moreover, $\varepsilon/T$ is equal to $2R_\lambda /\sqrt{T}$ and $G L \log(eT)/T$ in the convex and strongly convex cases respectively (given the $\eta_t$ from Lemma~1). 
Finally, the upper bound on $\Scal_\Ab$ is given by Lemma~4. % \ref{lem:regret_At}. 
\end{proof}
\section{Appendix: Technical lemmas}
\begin{lemma}\label{lmm1}
For any $\lambdab \in \Lambda$ and $\xb_t \in \Xcal_t$, we have
\begin{align}
\Lcal_t(\xb_t,\hlb_t) - \Lcal_t(\hxb_t,\lambdab) &\leq \Lcal_t(\hxb_t,\hlb_t)  - \Lcal_t(\hxb_t,\lambdab) \nonumber \\ &+ \hlb_t^\top(\hAb_t - \Ab_t )(\xb_t - \hxb_t)\nonumber.
\end{align}
\end{lemma}
\begin{proof}
For any $\lambdab \in \Lambda$ and $\xb_t \in \Xcal_t$, $\Lcal_t(\xb_t,\hlb_t) - \Lcal_t(\hxb_t,\lambdab)$ is equal to
\begin{eqnarray*}
&=&\!\!\!  \hat{\Lcal}_t(\xb_t,\hlb_t) + \hlb_t^\top(\hAb_t - \Ab_t )\xb_t - \Lcal_t(\hxb_t,\lambdab)\\
&\leq&\!\!\! \hat{\Lcal}_t(\hxb_t,\hlb_t) + \hlb_t^\top(\hAb_t - \Ab_t )\xb_t - \Lcal_t(\hxb_t,\lambdab) \\
&=&\!\!\! \Lcal_t(\hxb_t,\hlb_t)  - \Lcal_t(\hxb_t,\lambdab) + \hlb_t^\top(\hAb_t - \Ab_t )(\xb_t - \hxb_t),
\end{eqnarray*}
where we have used twice the relationship $\hat{\Lcal}_t(\xb_t,\lambdab) = \Lcal_t(\xb_t,\lambdab) - \lambdab^\top(\hAb_t - \Ab_t )\xb_t$, and the fact that $\hxb_t$ is defined as $\arg\!\max_{\xb_t \in \Xcal_t}  \Lcal_t(\xb_t, \hlb_t)$.
\end{proof}
We next provide the details of the computations related to relationship between $\Ecal$ and $\Ecal^*$ presented in Table~\ref{tab:ecal_compact_version}. We focus on one of the displayed instantiations, since the arguments in those other cases follow along the same lines.
\begin{lemma} \label{lem:ecal_ecalstar} For any $\ub \in \Real^m$,
$$
\max_{\wb\in\Real^m} \Big\{ \ub^\top \wb - \Ical_{\Bcal_r \cap \Real^m_+}(\wb)  - \frac{1}{2s} \|  \wb \|_2^2 \Big\} = \Hcal_{r,s}( \| [\ub]_+\|_2),
$$
where $\Hcal_{r,s}(t) \defin \frac{1}{2} \min\{s t^2,\frac{r^2}{s}\} + r [ |t| -\frac{r}{s} ]_+$.
\end{lemma}
\begin{proof}
We start by deriving the Lagrangian associated with
\begin{eqnarray*}
& \max_{\wb\in\Real^m}  \Big\{ \ub^\top \wb - \Ical_{\Bcal_r \cap \Real^m_+}(\wb)  - \frac{1}{2s} \|  \wb \|_2^2 \Big\} \\
= 
&\hspace{-1.8cm}  \max_{\wb\in\Bcal_r \cap \Real^m_+} \!\Big\{ \ub^\top \wb - \frac{1}{2s} \|  \wb \|_2^2 \Big\} 
\end{eqnarray*}
that is given by, for any $\alpha \geq 0, \betab \in \Real^m_+$,
$$
L(\wb,\alpha,\betab) = \ub^\top \wb + \frac{\alpha}{2}\big (r^2 -  \|\wb\|_2^2 \big) -  \frac{1}{2s} \|  \wb \|_2^2 + \betab^\top \wb.
$$
Maximizing out $\wb$, and recognizing the conjugate of the squared $\ell_2$ norm (Example 3.27 in \cite{Boyd2004}), we obtain the dual function
$$
\max_{\wb \in \Real^m} L(\wb,\alpha,\betab) = \frac{1}{2 (\alpha + 1/s)} \| \ub + \betab \|_2^2 + \frac{\alpha}{2} r^2.
$$
In turn, we minimize with respect to the dual variables
\begin{eqnarray*}
\min_{\alpha \geq 0, \betab \in \Real^m_+} \max_{\wb \in \Real^m} L(\wb,\alpha,\betab) = \\
\min_{\alpha \geq 0} \Big\{ \frac{1}{2 (\alpha + 1/s)} \| [\ub]_+ \|_2^2 + \frac{\alpha}{2} r^2\Big\} .
\end{eqnarray*}
The optimal $\alpha$ can then be easily computed and is equal to
$$
\Big[  \frac{\| [\ub]_+ \|_2}{r} - \frac{1}{s} \Big]_+.
$$
Plugging back this value into the dual function leads to the expression $\Hcal_{r,s}( \| [\ub]_+\|_2)$.
The equality holds by invoking strong duality, which applies based on Slater's constraint qualification~(Section 5.2.3 in \cite{Boyd2004}).
\end{proof}
\section{Appendix: Additional experiments for Section 4.1}
We show below additional results when $\{\Ab_t,\bb_t,\ub_t\}$ are generated according to different distributions, namely, Cauchy, uniform and gamma. As described in the protocol of Section 4.1, we continue to normalize $\{\Ab_t,\bb_t,\ub_t\}$ to unit norms.
The same conclusions as those explained in Section 4.1 of the paper hold. We observe that some instantiations of $\Ecal$ and distribution, e.g., $\ell_1$ with gamma in~Figure~\ref{fig:reward_vs_ecal_gamma}, lead to settings where the additive relaxation (12) appears as tighter, although our non-additive approach still offers better reward versus constraint violation tradeoffs.
\begin{figure}[h]
  \centering
  \includegraphics[width=0.5\textwidth]{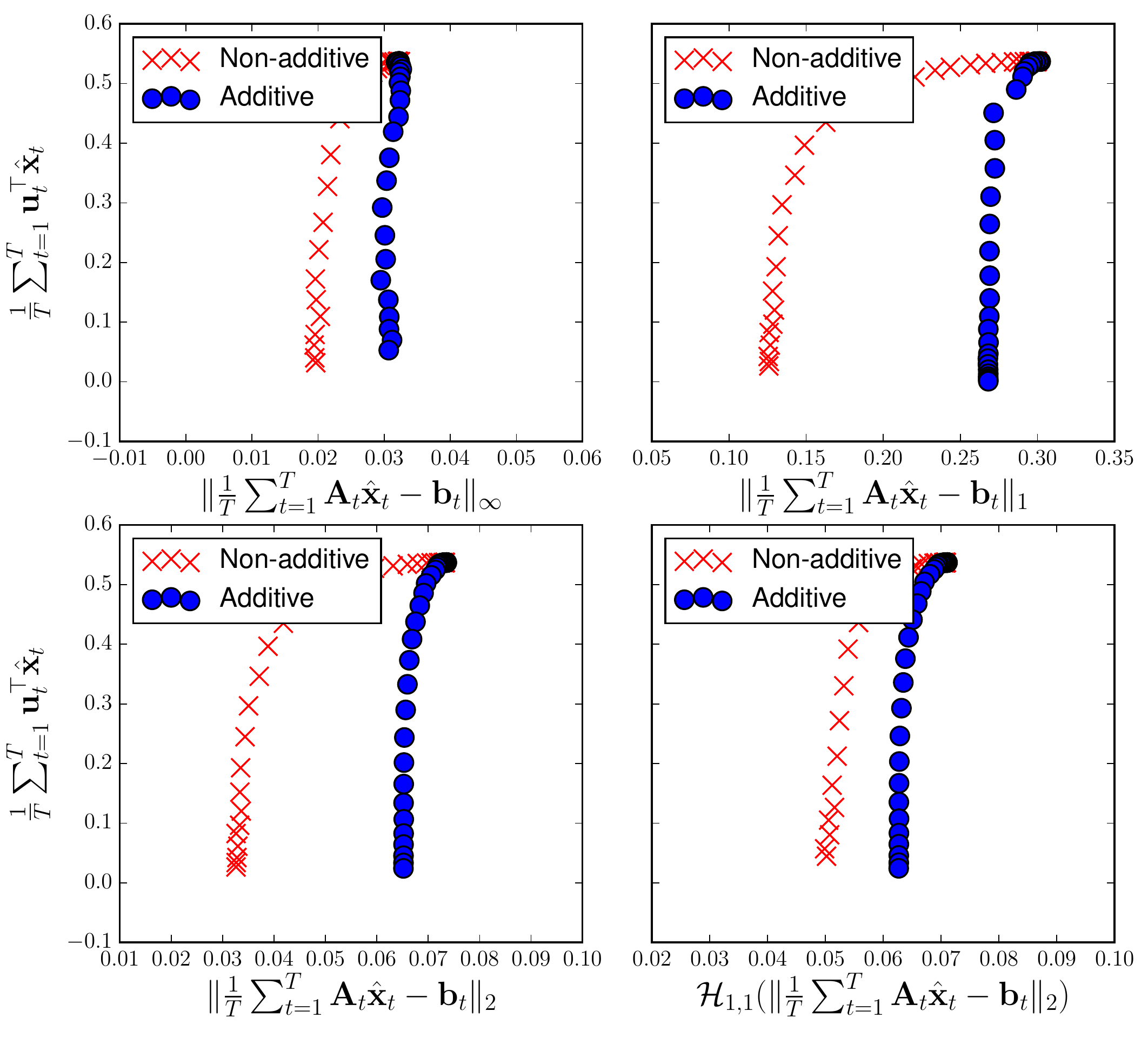}
  \caption{Reward $\frac{1}{T}\sum_t \ub_t^\top\hxb_t$ as a function of non-additive penalty $\Ecal(\frac{1}{T}\sum_t\mathbf{A}_t\hat{\mathbf{x}}_t-\mathbf{b}_t)/R_\lambda$ for 
  $R_\lambda = 2^{\gamma}$ with $\gamma \in \{-8,-7.5,\dots, 10\}$. The problem instances are generated with a standard Cauchy distribution.
  Red crosses correspond to our proposed online algorithm, while blue circles stand for a baseline algorithm with additive penalties.
  Each subplot displays a different instantiation of $\Ecal$, namely $\Ecal(\zb) = R_\lambda \cdot \| \zb \|_q$ for $q \in \{1,2,\infty\}$ and $\Ecal(\zb) = R_\lambda \cdot \Hcal_{1,1}(\|\zb\|_2)$ where $\Hcal_{1,1}$ is defined in Table~1.}\label{fig:reward_vs_ecal_cauchy}
\end{figure}
\begin{figure}[h]
  \centering
  \includegraphics[width=0.5\textwidth]{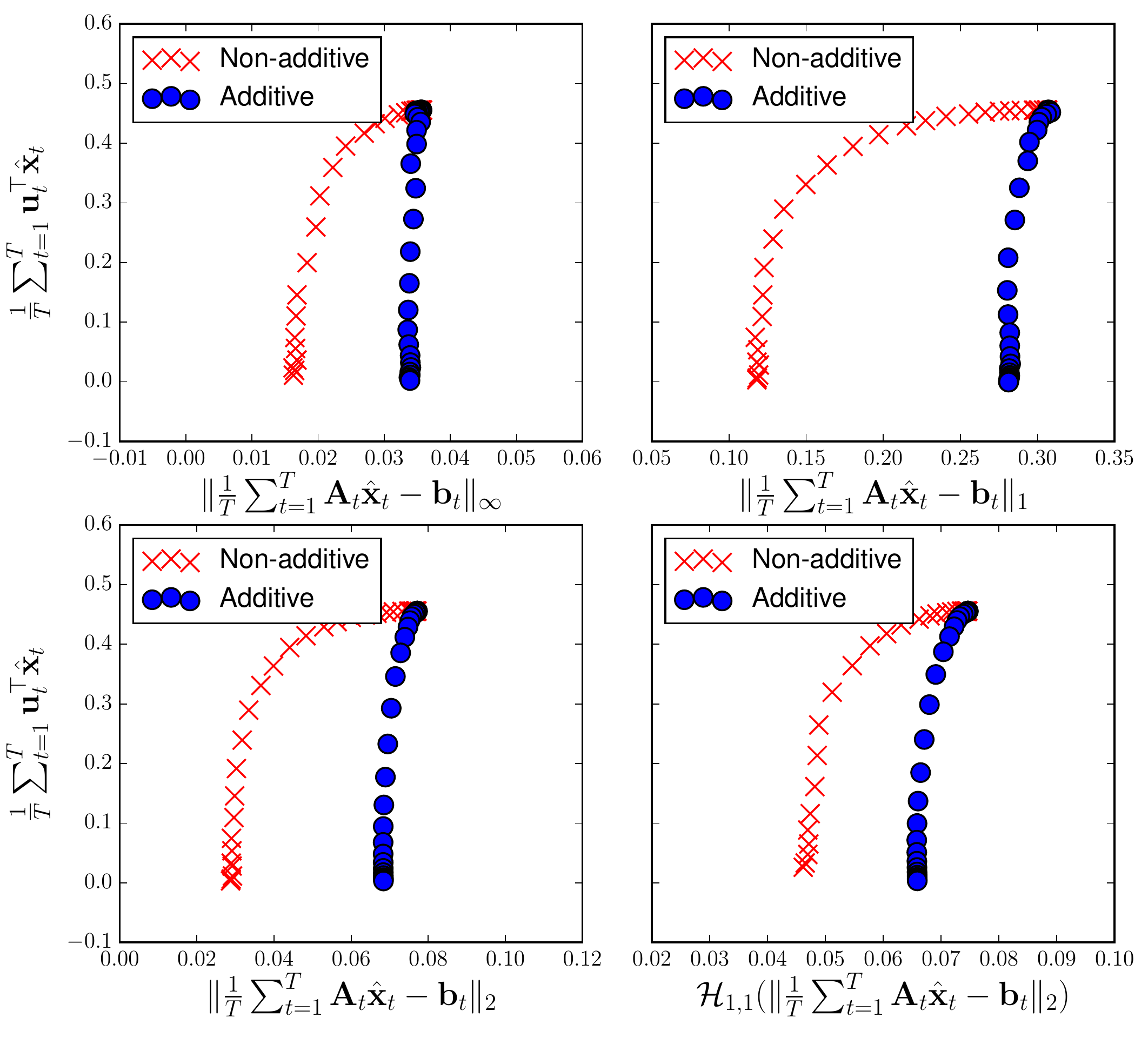}
  \caption{Reward $\frac{1}{T}\sum_t \ub_t^\top\hxb_t$ as a function of non-additive penalty $\Ecal(\frac{1}{T}\sum_t\mathbf{A}_t\hat{\mathbf{x}}_t-\mathbf{b}_t)/R_\lambda$ for 
  $R_\lambda = 2^{\gamma}$ with $\gamma \in \{-8,-7.5,\dots, 10\}$. The problem instances are generated with an uniform(-1,1) distribution.
  Red crosses correspond to our proposed online algorithm, while blue circles stand for a baseline algorithm with additive penalties.
  Each subplot displays a different instantiation of $\Ecal$, namely $\Ecal(\zb) = R_\lambda \cdot \| \zb \|_q$ for $q \in \{1,2,\infty\}$ and $\Ecal(\zb) = R_\lambda \cdot \Hcal_{1,1}(\|\zb\|_2)$ where $\Hcal_{1,1}$ is defined in Table~1.}\label{fig:reward_vs_ecal_uniform}
\end{figure}
\begin{figure}[h]
  \centering
  \includegraphics[width=0.5\textwidth]{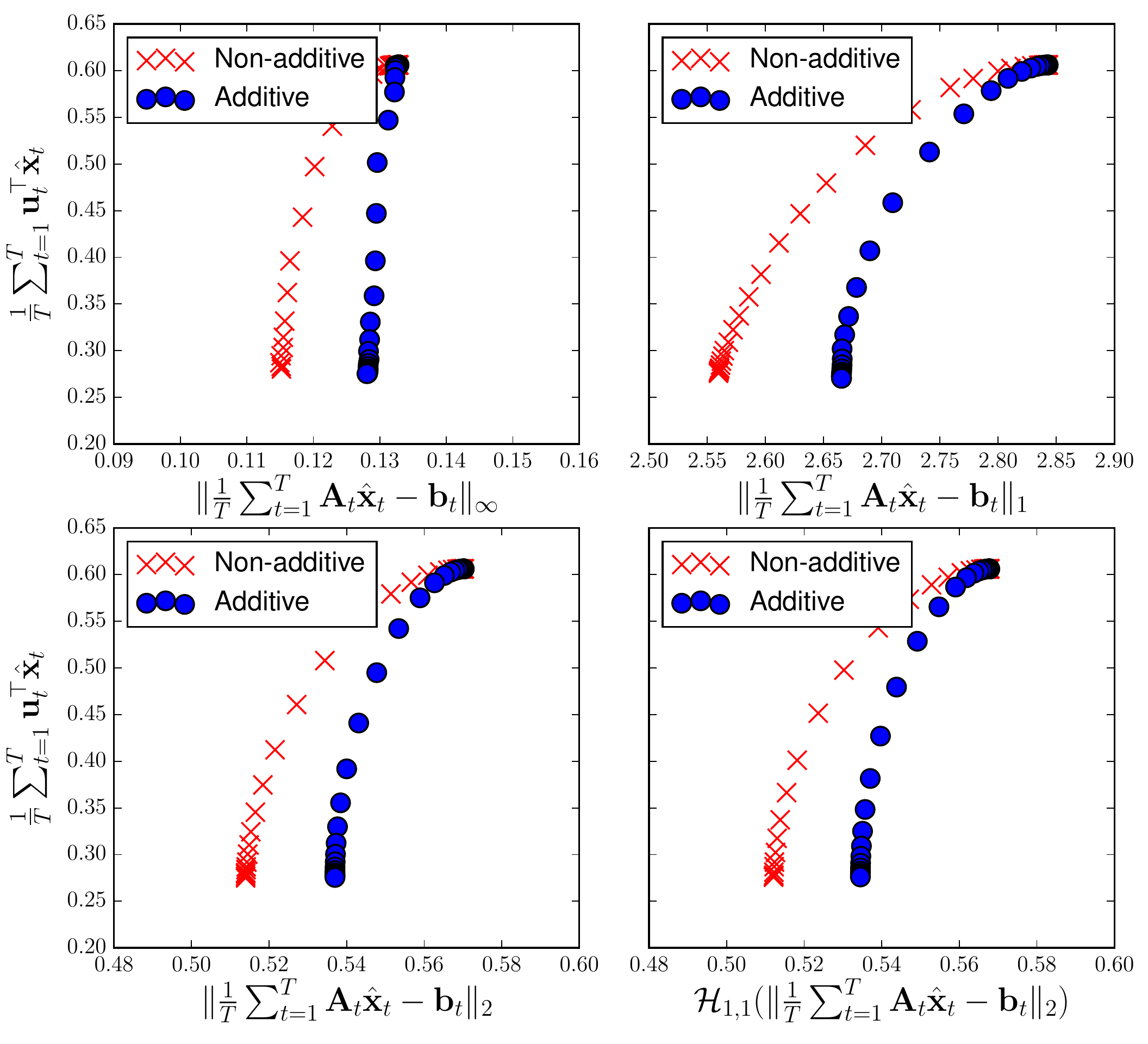}
  \caption{Reward $\frac{1}{T}\sum_t \ub_t^\top\hxb_t$ as a function of non-additive penalty $\Ecal(\frac{1}{T}\sum_t\mathbf{A}_t\hat{\mathbf{x}}_t-\mathbf{b}_t)/R_\lambda$ for 
  $R_\lambda = 2^{\gamma}$ with $\gamma \in \{-8,-7.5,\dots, 10\}$. The problem instances are generated with a gamma(2,2) distribution.
  Red crosses correspond to our proposed online algorithm, while blue circles stand for a baseline algorithm with additive penalties.
  Each subplot displays a different instantiation of $\Ecal$, namely $\Ecal(\zb) = R_\lambda \cdot \| \zb \|_q$ for $q \in \{1,2,\infty\}$ and $\Ecal(\zb) = R_\lambda \cdot \Hcal_{1,1}(\|\zb\|_2)$ where $\Hcal_{1,1}$ is defined in Table~1.}\label{fig:reward_vs_ecal_gamma}
\end{figure}

\bibliographystyle{apa}
\bibliography{bibliography}

\begin{thebibliography}{}

\bibitem[\protect\astroncite{Agrawal and Devanur}{2015}]{Agrawal2015}
Agrawal, S. and Devanur, N.~R. (2015).
\newblock Fast algorithms for online stochastic convex programming.
\newblock In {\em SODA 2015 (ACM-SIAM Symposium on Discrete Algorithms)}.
  SIAM-Society for Industrial and Applied Mathematics.

\bibitem[\protect\astroncite{Andrew et~al.}{2013}]{Andrew2013}
Andrew, L., Barman, S., Ligett, K., Lin, M., Meyerson, A., Roytman, A., and
  Wierman, A. (2013).
\newblock A tale of two metrics: simultaneous bounds on competitiveness and
  regret.
\newblock In {\em ACM SIGMETRICS Performance Evaluation Review}, volume~41,
  pages 329--330.

\bibitem[\protect\astroncite{Awerbuch and Kleinberg}{2008}]{Awerbuch2008}
Awerbuch, B. and Kleinberg, R. (2008).
\newblock Online linear optimization and adaptive routing.
\newblock {\em Journal of Computer and System Sciences}, 74(1):97--114.

\bibitem[\protect\astroncite{Bera et~al.}{2013}]{Bera2013}
Bera, S.~K., Choudhury, A.~R., Das, S., Roy, S., and Thatchachar, J.~S. (2013).
\newblock Fenchel duals for drifting adversaries.
\newblock Technical report, preprint arXiv:1309.5904.

\bibitem[\protect\astroncite{Borodin and El-Yaniv}{2005}]{Borodin2005}
Borodin, A. and El-Yaniv, R. (2005).
\newblock {\em Online computation and competitive analysis}.
\newblock Cambridge University Press.

\bibitem[\protect\astroncite{Borodin et~al.}{1992}]{Borodin1992}
Borodin, A., Linial, N., and Saks, M.~E. (1992).
\newblock An optimal on-line algorithm for metrical task system.
\newblock {\em Journal of the ACM (JACM)}, 39(4):745--763.

\bibitem[\protect\astroncite{Boyd and Vandenberghe}{2004}]{Boyd2004}
Boyd, S.~P. and Vandenberghe, L. (2004).
\newblock {\em Convex Optimization}.
\newblock Cambridge University Press.

\bibitem[\protect\astroncite{Buchbinder et~al.}{2012}]{Buchbinder2012}
Buchbinder, N., Chen, S., Naor, J., and Shamir, O. (2012).
\newblock Unified algorithms for online learning and competitive analysis.
\newblock In {\em Proceedings of the Annual Conference on Computational
  Learning Theory (COLT)}, pages 5--1.

\bibitem[\protect\astroncite{Cesa-Bianchi et~al.}{2012}]{Cesa-Bianchi2012}
Cesa-Bianchi, N., Gaillard, P., Lugosi, G., and Stoltz, G. (2012).
\newblock A new look at shifting regret.
\newblock Technical report, CoRR abs/1202.3323.

\bibitem[\protect\astroncite{Cesa-Bianchi and Lugosi}{2006}]{Cesa-Bianchi2006}
Cesa-Bianchi, N. and Lugosi, G. (2006).
\newblock {\em Prediction, learning, and games}.
\newblock Cambridge University Press.

\bibitem[\protect\astroncite{Chen et~al.}{2011}]{rtb}
Chen, Y., Berkhin, P., Anderson, B., and Devanur, N. (2011).
\newblock Online bidding algorithms for performance-based display ad
  allocation.
\newblock In {\em Proceedings of the 17th ACM SIGKDD International Conference
  on Knowledge Discovery and Data Mining (KDD 2011)}, pages 1307--1315.

\bibitem[\protect\astroncite{Diamond et~al.}{2014}]{Diamond2014}
Diamond, S., Chu, E., and Boyd, S. (2014).
\newblock {CVXPY}: A {P}ython-embedded modeling language for convex
  optimization, version 0.2.
\newblock \texttt{http://cvxpy.org/}.

\bibitem[\protect\astroncite{Even-Dar et~al.}{2009}]{Even-Dar2009}
Even-Dar, E., Kleinberg, R., Mannor, S., and Mansour, Y. (2009).
\newblock Online learning for global cost functions.
\newblock In {\em Proceedings of the annual conference on Computational
  Learning Theory (COLT)}.

\bibitem[\protect\astroncite{Hall and Willett}{2013}]{Hall2013}
Hall, E. and Willett, R. (2013).
\newblock Dynamical models and tracking regret in online convex programming.
\newblock In {\em Proceedings of the International Conference on Machine
  Learning (ICML)}, pages 579--587.

\bibitem[\protect\astroncite{Hastie et~al.}{2009}]{Hastie2009}
Hastie, T., Tibshirani, R., and Friedman, J. (2009).
\newblock {\em The Elements of Statistical Learning: Data Mining, Inference,
  and Prediction, Second Edition}.
\newblock Springer.

\bibitem[\protect\astroncite{Hazan et~al.}{2007}]{Hazan2007}
Hazan, E., Agarwal, A., and Kale, S. (2007).
\newblock Logarithmic regret algorithms for online convex optimization.
\newblock {\em Machine Learning}, 69(2-3):169--192.

\bibitem[\protect\astroncite{Helmbold et~al.}{1996}]{portfolio}
Helmbold, D., Schapire, R., Singer, Y., and Warmuth, M. (1996).
\newblock On-line portfolio selection using multiplicative updates.
\newblock In {\em Proceedings of the International Conference on Machine
  Learning (ICML)}, pages 243--251.

\bibitem[\protect\astroncite{Hsu et~al.}{2012}]{SubgaussianQuadratic}
Hsu, D., Kakade, S.~M., and Zhang, T. (2012).
\newblock A tail inequality for quadratic forms of subgaussian random vectors.
\newblock {\em Electron. Commun. Probab}, 17(52):1--6.

\bibitem[\protect\astroncite{Jadbabaie et~al.}{2015}]{Jadbabaie2015}
Jadbabaie, A., Rakhlin, A., Shahrampour, S., and Sridharan, K. (2015).
\newblock Online optimization: Competing with dynamic comparators.
\newblock Technical report, preprint arXiv:1501.06225.

\bibitem[\protect\astroncite{Jaggi}{2013}]{Jaggi2013}
Jaggi, M. (2013).
\newblock Revisiting {F}rank-{W}olfe: Projection-free sparse convex
  optimization.
\newblock In {\em Proceedings of the International Conference on Machine
  Learning (ICML)}, pages 427--435.

\bibitem[\protect\astroncite{Kar et~al.}{2014}]{Kar2014}
Kar, P., Narasimhan, H., and Jain, P. (2014).
\newblock Online and stochastic gradient methods for non-decomposable loss
  functions.
\newblock In {\em Advances in Neural Information Processing Systems}, pages
  694--702.

\bibitem[\protect\astroncite{Koppel et~al.}{2014}]{Koppel2014}
Koppel, A., Jakubiec, F.~Y., and Ribeiro, A. (2014).
\newblock A saddle point algorithm for networked online convex optimization.
\newblock In {\em Acoustics, Speech and Signal Processing (ICASSP), 2014 IEEE
  International Conference on}, pages 8292--8296. IEEE.

\bibitem[\protect\astroncite{Mahdavi et~al.}{2012}]{Mahdavi2012}
Mahdavi, M., Jin, R., and Yang, T. (2012).
\newblock Trading regret for efficiency: online convex optimization with long
  term constraints.
\newblock {\em Journal of Machine Learning Research}, 13(1):2503--2528.

\bibitem[\protect\astroncite{Parikh and Boyd}{2013}]{Parikh2013}
Parikh, N. and Boyd, S. (2013).
\newblock Proximal algorithms.
\newblock {\em Foundations and Trends in optimization}, 1(3):123--231.

\bibitem[\protect\astroncite{Rakhlin et~al.}{2010}]{Rakhlin2010}
Rakhlin, A., Sridharan, K., and Tewari, A. (2010).
\newblock Online learning: Beyond regret.
\newblock Technical report, preprint arXiv:1011.3168.

\bibitem[\protect\astroncite{Shalev-Shwartz}{2011}]{Shalev-Shwartz2011}
Shalev-Shwartz, S. (2011).
\newblock Online learning and online convex optimization.
\newblock {\em Foundations and Trends in Machine Learning}, 4(2):107--194.

\bibitem[\protect\astroncite{Shalev-Shwartz and
  Kakade}{2009}]{Shalev-Shwartz2009}
Shalev-Shwartz, S. and Kakade, S.~M. (2009).
\newblock Mind the duality gap: Logarithmic regret algorithms for online
  optimization.
\newblock In {\em Advances in Neural Information Processing Systems}, pages
  1457--1464.

\bibitem[\protect\astroncite{Shalev-Shwartz and
  Singer}{2006}]{Shalev-Shwartz2006}
Shalev-Shwartz, S. and Singer, Y. (2006).
\newblock Convex repeated games and {F}enchel duality.
\newblock In {\em Advances in Neural Information Processing Systems}, pages
  1265--1272.

\bibitem[\protect\astroncite{Sion et~al.}{1957}]{Sion1957}
Sion, M., of~Scientific~Research, O., States, U., and Force, A. (1957).
\newblock {\em General Minimax Theorems}.
\newblock United States Air Force, Office of Scientific Research.

\bibitem[\protect\astroncite{Usmani}{1994}]{TridiagonalInversion}
Usmani, R.~A. (1994).
\newblock Inversion of a tridiagonal jacobi matrix.
\newblock {\em Linear Algebra and its Applications}, 212:413--414.

\bibitem[\protect\astroncite{Zinkevich}{2003}]{Zinkevich2003}
Zinkevich, M. (2003).
\newblock Online convex programming and generalized infinitesimal gradient
  ascent.
\newblock In {\em Proceedings of the International Conference on Machine
  Learning (ICML)}.

\end{thebibliography}
\end{document}